\documentclass{article}

\usepackage{amsthm}
\usepackage{thm-restate}
\usepackage{booktabs}
\usepackage{framed}
\usepackage{color}
\usepackage{xcolor}
\definecolor{mydarkblue}{rgb}{0,0.08,0.45}
\usepackage[hidelinks,colorlinks]{hyperref}
\hypersetup{
	colorlinks=true,
	linkcolor=mydarkblue,
	citecolor=mydarkblue,
	filecolor=mydarkblue,
	urlcolor=mydarkblue,
}
\usepackage[ruled]{algorithm2e}
\SetKwProg{Fn}{function}{}{end}
\SetKwProg{Init}{Initialization}{}{}
\SetKwProg{FirstPass}{First Pass:}{}{}
\SetKwProg{ZeroPass}{Zeroth Pass:}{}{}
\SetKwProg{SecondPass}{Second Pass:}{}{}
\SetKwProg{ThirdPass}{Third Pass:}{}{}
\SetKwProg{SVD}{SVD:}{}{}
\SetKwInput{Input}{Input}
\SetKw{KwTo}{in}
\SetKw{KwWhere}{where}
\usepackage{thmtools,thm-restate}
\newtheorem{theorem}{Theorem}
\newtheorem{corollary}{Corollary}
\newtheorem{lemma}{Lemma}

\newtheorem{assumption}{Assumption}

\theoremstyle{definition}
\newtheorem{definition}[theorem]{Definition}
\newtheorem*{theorem*}{Theorem}

\usepackage{times}

\usepackage{enumerate}

\usepackage{graphicx}
\usepackage[numbers]{natbib}
\usepackage{mathrsfs,amssymb,amsmath,textcomp,amsfonts,bbm,fullpage}

\usepackage{color}
\usepackage{multirow,amsmath, dsfont, amscd, latexsym,color,amssymb,setspace}
\usepackage{ifthen}
\usepackage{enumerate}

\usepackage{listings}
\usepackage{enumitem}
\usepackage{etoolbox}
\usepackage{caption}
\usepackage{thmtools,thm-restate}
\usepackage{comment}
\usepackage{subcaption}
\usepackage{url}
\usepackage{graphicx,caption, subcaption, multirow}
\usepackage{float}

\providetoggle{solution}
\settoggle{solution}{false}

\def\eb{\mathbf{e}}

\newcommand{\TODO}[1]{}
\newcommand{\td}[1]{\ensuremath{\tilde{#1}}}

\newcommand{\rowa}[1]{\ensuremath{a_{(#1)}}}

\newcommand{\cola}[1]{\ensuremath{a^{(#1)}}}
\newcommand{\colof}[2]{\ensuremath{#1^{(#2)}}}
\newcommand{\sqnorm}[1]{\ensuremath{\|#1\|_2^2}}

\newcommand{\remove}[1]{}

\newcommand{\Sigmab}{\boldsymbol{\Sigma}}
\newcommand{\Lambdab}{\boldsymbol{\Lambda}}
\newcommand{\Omegab}{\boldsymbol{\Omega}}
\newcommand{\Xb}{\mathbf{X}}

\newcommand{\Ub}{\mathbf{U}}
\newcommand{\Vb}{\mathbf{V}}
\newcommand{\Qb}{\mathbf{Q}}
\newcommand{\Rb}{\mathbf{R}}
\newcommand{\Ab}{\mathbf{A}}

\newcommand{\Nb}{\mathbf{N}}
\newcommand{\Pb}{\mathbf{P}}

\newcommand{\Cb}{\mathbf{C}}
\newcommand{\Db}{\mathbf{D}}
\newcommand{\Ib}{\mathbf{I}}

\newcommand{\norm}[1]{\| #1 \|}
\newcommand{\defeq}{:=}
\newcommand{\rank}{\text{rank}}
\newcommand{\normfsq}[1]{\| #1 \|_{\normalfont\text{F}}^2}

\newtheorem*{Conj*}{Conjecture}

\newcommand{\eat}[1]{}

\newcommand{\commented}{no}

\ifthenelse{\equal{\commented}{yes}}{
\newcommand{\pnote}[1]{\footnote{{\bf [[Parik: {#1}\bf ]] }}}
}

\ifthenelse{\equal{\commented}{no}}{
\newcommand{\pnote}[1]{}
}

\newcommand{\ignore}[1]{}

\def\colorful{3}

\ifnum\colorful=1

\fi

\newcommand{\diag}{\mathrm{diag}}
\newcommand{\sr}{\mathrm{sr}}
\newcommand{\rk}{\mathrm{rank}}

\newcommand{\R}{\mathbb R}

\newcommand{\eps}{\varepsilon}

\newcommand{\poly}{\mathrm{poly}}

\newcounter{this-list}

\usepackage[nonatbib,final]{nips_2018}
\usepackage{pbox}
\usepackage{makecell}

\usepackage[utf8]{inputenc} 
\usepackage[T1]{fontenc}    
\usepackage{hyperref}       
\usepackage{url}            
\usepackage{booktabs}       
\usepackage{amsfonts}       
\usepackage{nicefrac}       
\usepackage{microtype}      

\title{Efficient Anomaly Detection via Matrix Sketching}

\author{
	Vatsal Sharan \\
	Stanford University\thanks{Part of the work was done while the author was an intern at VMware Research.} \\
	\fontsize{10}{10}\selectfont {\tt vsharan@stanford.edu}
	\And
	Parikshit Gopalan \\
	VMware Research \\
	\fontsize{10}{10}\selectfont {\tt pgopalan@vmware.com}
	\And
	Udi Wieder \\
	VMware Research \\
	\fontsize{10}{10}\selectfont {\tt uwieder@vmware.com}
}

\begin{document}

\maketitle
\begin{abstract}
We consider the problem of finding anomalies in high-dimensional data using popular PCA based anomaly scores. 
The naive algorithms for computing these scores explicitly compute the PCA of the covariance matrix which
uses space quadratic in the dimensionality of the data. We give the
first streaming algorithms that use space that is linear or sublinear in the dimension. We prove general results showing that \emph{any} sketch of a matrix that satisfies a certain operator
norm guarantee can be used to approximate these scores. We
instantiate these results with powerful matrix sketching techniques
such as Frequent Directions and random projections to
derive efficient and practical algorithms for these problems, which we validate over real-world data sets.
Our main technical contribution is to prove matrix perturbation
inequalities for operators arising in the computation of these measures.

\eat{
	We present efficient streaming algorithms to find anomalies in high dimensional data for two commonly anomaly scores: the rank-$k$ leverage scores (aka Mahalanobis distance) and the rank-$k$ projection distance. 
	The naive algorithms for these tasks,
	based on computing the covariance matrix, and then computing the SVD
	use space quadratic in the dimensionality of the data. We give the
	first streaming algorithms that use space that is linear or sublinear in the dimension. We prove general results showing that \emph{any} sketch of a matrix that satisfies a certain operator
	norm guarantee can be used to approximate these measures. We
	instantiate these results with powerful matrix sketching techniques
	such as the Frequent Directions sketch and random projections to
	derive efficient and practical algorithms for these problems.
	Our main technical contribution is to prove matrix perturbation
	inequalities for operators arising in the computation of these measures.
}

\end{abstract}



\section{Introduction}


Anomaly detection in high-dimensional numeric data is a ubiquitous
problem in machine learning \citep{Aggarwal2013,ChandolaBK09}. A typical scenario is where we have a
constant stream of measurements (say parameters regarding the
health of machines in a data-center), and our goal is to detect any unusual
behavior. An algorithm to detect anomalies in such high dimensional settings faces computational
challenges: the dimension of the data matrix $\Ab \in \mathbb{R}^{n \times d}$ may be very large both
in terms of the number of data points $n$ and their dimensionality $d$
(in the datacenter example, $d$ could be $10^6$ and $n \gg d$).
The desiderata for an algorithm to be efficient in such settings are---

1. As $n$ is too large for the data to be stored in memory, the algorithm must
work in a streaming fashion where it only gets a constant
number of passes over the dataset.\\
2. As $d$ is also very large, the algorithm should ideally use memory
linear or even sublinear in $d$.


In this work we focus on two popular subspace based anomaly
scores: rank-$k$ leverage scores and rank-$k$ projection distance. The
key idea behind subspace based anomaly scores is that real-world data
often has most of its variance in a low-dimensional rank $k$ subspace,
where $k$ is usually much smaller than $d$. In this section, we assume
$k = O(1)$ for simplicity. These scores are based on identifying this
principal $k$ subspace using Principal Component Analyis (PCA) and
then computing how ``normal'' the projection of a point on the
principal $k$ subspace looks. Rank-$k$ leverage scores compute the
normality of the projection of the point \emph{onto} the principal $k$
subspace using Mahalanobis distance, and rank-$k$ projection distance
compute the $\ell_2$ distance of the point \emph{from} the principal
$k$ subspace (see Fig. \ref{fig:lev} for an
illustration). These scores have found widespread use for detection of
anomalies in many applications such as finding outliers in network
traffic data
\citep{lakhina2004diagnosing,lakhina2005mining,huang2007communication,huang2007network},
detecting anomalous behavior in social networks
\citep{viswanath2014towards,Portnoff2018thesis},  intrusion detection
in computer security
\citep{shyu2003novel,wang2004novel,davis2011data}, in industrial
systems for fault detection
\citep{chiang2000fault,russell2000fault,joe2003statistical} and for
monitoring data-centers \citep{xu2009detecting,mi2013toward}.

	\begin{figure}[t]
		\centering
		\includegraphics[width=2 in]{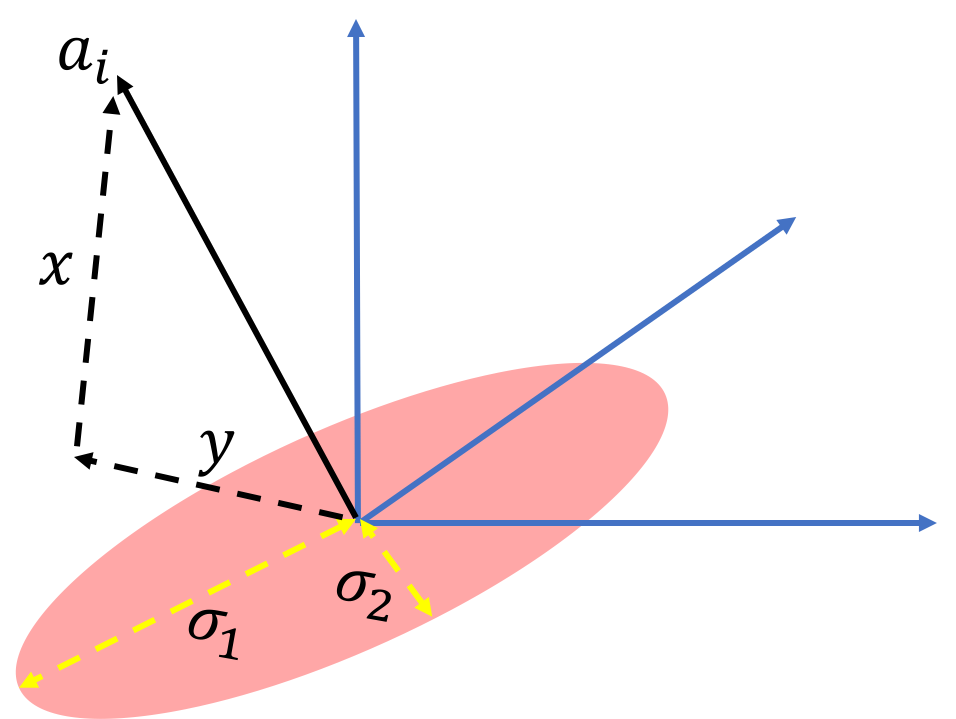}
		\caption{Illustration of subspace based anomaly
                  scores. Here, the data lies mostly in the $k=2$
                  dimensional principal subspace shaded in red. For a
                  point $a_{(i)}$ the rank-$k$ projection distance
                  equals $\norm{x}_2$, where $x$ is the component of
                  $a_{(i)}$ orthogonal to the principal subspace. The
                  rank-$k$ leverage score measures the {\em normality}
                  of the projection $y$ onto the principal subspace.}
		\label{fig:lev}
	\end{figure}

The standard approach to compute principal $k$ subspace based anomaly
scores in a streaming setting is by computing $\Ab^T\Ab$, the
$(d\times d)$ covariance matrix of the data, and then computing
the top $k$ principal components. This takes space $O(d^2)$ and time
$O(nd^2)$. The quadratic dependence on $d$ renders this approach inefficient
in high dimensions. It raises the natural question of whether better
algorithms exist.

\subsection{Our Results}
In this work, we answer the above question affirmatively, by giving
algorithms for computing these anomaly scores that require space
linear and even sublinear in $d$. Our algorithms use
popular matrix sketching techniques while their analysis uses new matrix perturbation
inequalities that we prove. Briefly, a sketch of a matrix produces a
much smaller matrix that preserves some desirable properties of the
large matrix (formally, it is close in some suitable norm).
Sketching techniques have found numerous applications to numerical
linear algebra. Several efficient sketching algorithms
are known in the streaming setting \citep{woodruff2014sketching}.

\paragraph{Pointwise guarantees with linear space: }We show that any sketch $\tilde{\Ab}$ of $\Ab$ with the property that
$\norm{\Ab^T\Ab-\tilde{\Ab}^T\tilde{\Ab}}$ is small, can be used to
additively approximate the rank-$k$ leverage scores and rank-$k$ projection
distances for each row. By instantiating this with suitable sketches such as the Frequent
Directions sketch \citep{liberty2013simple}, row-sampling
\citep{drineas2006fast} or a random projection of the columns of the input, we get a
streaming algorithm that uses $O(d)$ memory and $O(nd)$ time.

\paragraph{A matching lower bound: } Can we get such an additive approximation using
memory only $o(d)$?\footnote{Note
		that even though each row is $d$ dimensional an algorithm need not
		store the entire row in memory, and could instead perform
		computations as each coordinate of the row streams
                in.}
The answer is no, we show a lower bound saying that any algorithm that computes
such an approximation to  the rank-$k$ leverage scores or the rank-$k$
projection distances for all the rows of a matrix {must} use
$\Omega(d)$ working space, using techniques from communication complexity. Hence our algorithm has near-optimal dependence on $d$
for the task of approximating the outlier scores for every data point.

\paragraph{Average-case guarantees with logarithmic space: }
Perhaps surprisingly, we show that it is actually possible to circumvent the lower bound by relaxing the requirement that the outlier
scores be preserved for each and every point to only preserving the outlier
scores on average. For this we require sketches where
${\norm{\Ab\Ab^T-\tilde{\Ab}\tilde{\Ab}^T}}$ is small: this can be
achieved via random projection of the rows of the input matrix or column
subsampling \citep{drineas2006fast}. Using any such sketch, we give a
streaming algorithm that
can preserve the outlier scores for the rows up to small additive error on average, and hence
preserve most outliers.
The space required by this algorithm is only $\poly(k)\log(d)$, and hence we get significant space savings in this setting (recall that we assume $k = O(1)$).


\paragraph{Technical contributions.}
A sketch of a matrix $\Ab$ is a significantly smaller matrix
$\tilde{\Ab}$  which approximates it well in some norm, say for
instance ${\norm{\Ab^T\Ab-\tilde{\Ab}^T\tilde{\Ab}}}$ is small. We can
think of such a sketch as a noisy approximation of the true
matrix.  In order to use such sketches for anomaly detection, we need
to understand how the noise affects the anomaly scores of the rows of the matrix.  Matrix perturbation theory studies the effect of adding noise
to the spectral properties of a matrix, which makes it the natural tool
for us. The basic results here include
Weyl's inequality \citep{horn1994} and Wedin's
theorem~\citep{Wedin1972}, which respectively give such bounds for eigenvalues and
eigenvectors. We use these results to derive perturbation bounds on
more complex projection operators that arise while computing outlier
scores, these operators involve projecting onto the
top-$k$ principal subspace, and rescaling each co-ordinate by some
function of the corresponding singular values. We believe these
results could be of independent interest.

\paragraph{Experimental results.}
Our results have a parameter $\ell$ that controls the size
and the accuracy of the sketch. While our theorems imply that $\ell$
can be chosen independent of $d$, they depend polynomially on $k$, the
desired accuracy and other parameters, and are probably pessimistic. We validate both our
algorithms on real world data. In our experiments, we found that choosing
$\ell$ to be a small multiple of $k$ was sufficient to get good
results. Our results show that one can get outcomes comparable to running full-blown SVD
using sketches which are significantly smaller in memory footprint,
faster to compute and easy to implement (literally a few lines of
Python code).

This contributes to a line of work that aims to make SVD/PCA scale to
massive datasets \citep{halko2011finding}. We give simple and practical
algorithms for anomaly score computation, that give SVD-like guarantees at a
significantly lower cost in terms of memory, computation and
communication.

\paragraph{Outline.} 
We describe the setup and define our anomaly scores in Section \ref{sec:setup}. We state our theoretical results in Section \ref{sec:results}, and the results of our experimental evaluations in Section \ref{sec:experiments}. We review related work in Section \ref{sec:related}. The technical part begins with Section \ref{sec:prelims} where we state and prove our matrix perturbation bounds. We use these bounds to get point-wise
approximations for outlier scores in Section \ref{sec:point}, and our
average-case approximations in Section \ref{sec:avg}. We prove our lower
bound in Section \ref{sec:lower_bound}. Missing proofs are deferred to
the Appendix.



\eat{
For instance,
clearly a data point which is unusual in magnitude should be
considered anomalous. A more subtle notion of anomalies are those that
stem from changes in the correlation  across dimensions. Consider for
instance the case where the CPU usage of a machine rises by
$10\%$. That in itself may not be considered an anomaly, however when
a large fraction of machines in the data center see this behavior
simultaneously an alarm should be issued. Another more subtle scenario
is when, say, a group of servers that are typically correlated
suddenly demonstrate uncorrelated changes in load. Detecting these
types of behaviors requires a more sophisticated way of assigning an
\emph{anomaly score} to a data point. We focus on two commonly used scores.

The Mahalanobis distance~\citep{Mahalanobis36} is well known  and
heavily used way to score this type of anomaly\footnote{See
  \citet{Wilks63,Hawkins74} for classic uses of the Mahalanobis
  distance in statistical outlier detection.}. It relies on the
Singular Value Decomposition of the data matrix, the rows of the
matrix are data points and columns are types of scorements.
}

\eat{
\paragraph{Outline.} We state and prove our matrix perturbation bounds in
Section \ref{sec:prelims}. We use these bounds to get point-wise
approximations for outlier scores in Section \ref{sec:point}, and our
average approximations in Section \ref{sec:avg}. We state our lower
bound in Section \ref{sec:lower_bound}. Missing proofs are deferred to
the Appendix.
}

\vspace{-6pt}

\section{Notation and Setup}\label{sec:setup}
\vspace{-4pt}
Given a matrix $\Ab \in
\R^{n \times d}$, we let $\rowa{i} \in \R^d$ denote its
$i^{th}$ row and $\cola{i} \in \R^n$ denote its $i^{th}$ column. Let $\Ub \Sigmab \Vb^T$ be the SVD of $\Ab$ where
$\Sigmab = \diag(\sigma_1, \ldots, \sigma_d)$, for $\sigma_1 \geq
\cdots \geq \sigma_d > 0$. Let $\kappa_k$ be the condition number  of the top $k$ subspace of $\Ab$, defined as $\kappa_k=\sigma_1^2/\sigma_k^2$. We
consider all vectors as column vectors (including $\rowa{i}$). We denote by $\|\Ab\|_{\text{F}}$ the Frobenius norm of $\Ab$, and by  $\|\Ab\|$ the operator norm or the largest singular value. 
Subspace based measures of anomalies have their origins in a classical metric in statistics known as Mahalanobis distance, denoted by $L(i)$ and defined as,
\vspace{-6pt}
\begin{align}
L(i) = \sum_{j=1}^d {(\rowa{i}^T\colof{v}{j})^2}/{\sigma^2_j},\label{eq:full_lev}
\end{align}
where $\rowa{i}$ and $\colof{v}{i}$ are the $i^{th}$ row of $\Ab$ and
$i^{th}$ column of $\Vb$ respectively. 
$L(i)$ is also known as the  \emph{leverage score} \citep{Drineas12,Mahoney11}. If the data is drawn from a multivariate Gaussian distribution, then $L(i)$ is proportional to the negative log likelihood of the data point, and hence is the right anomaly metric in this case. Note that the higher leverage scores correspond to outliers in the data.

However, $L(i)$  depends on the entire spectrum of singular values and is highly
 sensitive to smaller singular values, whereas real world data sets often have most of their signal in the top singular values. 
Therefore the above sum is often limited to
only the $k$ largest singular values (for some appropriately chosen
$k \ll d$) \citep{Aggarwal2013, Holgersson12}.  This measure is called
the rank $k$ leverage score $L^k(i)$, where
\vspace{-6pt}
\begin{align*}
L^k(i) = \sum_{j=1}^k {(\rowa{i}^T\colof{v}{j})^2}/{\sigma^2_j}. \end{align*}
The rank $k$ leverage score is concerned with the mass which lies
within the principal space, but to catch anomalies that are far from the principal subspace a second measure of anomaly is the
rank $k$ \emph{projection distance} $T^k(i)$, which is simply the distance of
the data point $a_{(i)}$ to the rank $k$ principal subspace---
\vspace{-6pt}
\begin{align*}
 T^k(i) = \sum_{j = k + 1}^d (\rowa{i}^T\colof{v}{j})^2.
 \end{align*}
Section~\ref{sec:ridge} in the Appendix has more discussion about a related anomaly score (the ridge leverage score~\citep{Aggarwal2013}) and how it relates to the above scores.
\paragraph{Assumptions.} We now discuss assumptions needed for our anomaly scores to be meaningful.

\emph{(1) Separation assumption.} If there is degeneracy in the spectrum
of the matrix, namely that  $\sigma_k^2 = \sigma_{k+1}^2$ then the
$k$-dimensional principal subspace is not unique, and then
the quantities $L^k$ and $T^k$ are not well defined, since their value will
depend on the choice of principal subspace. This suggests that we are
using the {\em wrong} value of $k$, since the choice of $k$ ought to be
such that the directions orthogonal to the principal subspace have
markedly less variance than those in the principal subspace.
Hence  we require that $k$ is such that there
is a gap in the spectrum at $k$.
\begin{assumption}
	We define a matrix $\Ab$ as being $(k, \Delta)$-separated if
	$
	\sigma_k^2 - \sigma^2_{k+1} \geq \Delta \sigma_1^2.
	$
	Our results assume that the data are $(k,\Delta)$-separated for
	$\Delta > 0$.
\end{assumption}
This assumptions manifests itself as an inverse polynomial dependence on
$\Delta$ in our bounds. This dependence is probably pessimistic: in
our experiments, we have found our algorithms do well on datasets which are
not degenerate, but where the separation $\Delta$ is not particularly
large.

\emph{(2) Approximate low-rank assumption.}
We assume that the top-$k$ principal subspace captures a constant
fraction (at least $0.1$) of the total variance in the data, formalized as follows.
\begin{assumption}\label{assume:low_rank}
	We assume the matrix $\Ab$ is approximately rank-$k$, i.e.,
	$
	\sum_{i=1}^{k}\sigma_i^2 \ge (1/10)\sum_{i=1}^{d}\sigma_i^2
	$.
\end{assumption}
From a technical standpoint, this assumption is not strictly needed: if Assumption \ref{assume:low_rank} is not true, our
results still hold, but in this case they depend on the stable rank
$\sr(\Ab)$ of $\Ab$, defined as
$\sr(\Ab)=\sum_{i=1}^{d}\sigma_i^2/\sigma_1^2$ (we state these
general forms of our results in Sections \ref{sec:point} and \ref{sec:avg}).

From a practical standpoint though, this assumption captures the setting
where the scores $L^k$ and $T^k$, and our guarantees are most
meaningful. Indeed, our experiments suggest that our algorithms work
best on data sets where relatively few principal
components explain most of the variance.


\paragraph{Setup.}

We work in the row-streaming model, where rows appear one after the other in time. Note that the leverage score of a row depends on the entire matrix, and hence computing the anomaly scores in the streaming model requires care, since if the rows are seen in streaming order, when row $i$ arrives we cannot compute its leverage score without seeing the rest of the input. Indeed, 1-pass algorithms are not possible
(unless they output the entire matrix of scores at the end of the pass, which clearly requires a lot of memory). Hence we will aim for 2-pass algorithms.

Note that there is a simple $2$-pass algorithm which uses $O(d^2)$ memory to compute the covariance matrix in one pass, then computes its SVD, and using this computes $L^k(i)$ and $T^k(i)$ in a second pass using memory $O(dk)$. This requires $O(d^2)$ memory and $O(nd^2)$ time, and our goal would be to reduce this to linear or sublinear in $d$.

Another reasonable way to define leverage scores and projection distances in the streaming model is to define them with respect to only the input seen so far. 
We refer to this as the online scenario, and refer to these scores as the online scores. Our result for sketches which preserve row spaces also hold in this online scenario. We defer more discussion of this online scenario to Section \ref{sec:online}, and focus here only on the scores defined with respect to the entire matrix for simplicity.


\section{Guarantees for anomaly detection via sketching}\label{sec:results}

Our main results say that given $\mu > 0$ and a $(k, \Delta)$-separated matrix  $\Ab
\in \mathbb{R}^{n \times d}$ with top singular value $\sigma_1$, any sketch $\td{\Ab}\in \mathbb{R}^{\ell \times d}$
satisfying
\begin{align}
\norm{\Ab^T\Ab - \td{\Ab}^T\td{\Ab}} &\leq \mu \sigma_1^2, \label{eq:guarantee1}
\end{align}
or a sketch $\td{\Ab}\in \mathbb{R}^{n \times \ell}$ satisfying
\begin{align}
 \norm{\Ab\Ab^T - \td{\Ab}\td{\Ab}^T} &\leq \mu \sigma_1^2, \label{eq:guarantee2}
\end{align}
can be used to approximate rank $k$ leverage scores and
the projection distance from the principal $k$-dimensional
subspace. The quality of the approximation depends on $\mu$, the separation $\Delta$, $k$ and the condition number $\kappa_k$ of the top $k$ subspace.\footnote{The dependence on $\kappa_k$ only appears for showing guarantees for rank-$k$ leverage scores $L^k$ in Theorem \ref{thm:FD}.} 
In order for the sketches to be useful, we also need them to be efficiently computable in a streaming fashion. We show how to use such sketches to design efficient
algorithms for finding anomalies in a streaming fashion using small
space and with fast running time.  The actual guarantees (and the proofs) for the two cases are
different and incomparable. This is to be expected as the
sketch guarantees are very different in the two cases:
Equation \eqref{eq:guarantee1} can be viewed as an approximation to the
covariance matrix of the row vectors, whereas Equation
\eqref{eq:guarantee2} gives an approximation for the covariance matrix
of the column vectors. Since the corresponding sketches can be viewed
as preserving the row/column space of $\Ab$ respectively, we will
refer to them as row/column space approximations.

\paragraph{Pointwise guarantees from row space approximations.}

Sketches which satisfy Equation \eqref{eq:guarantee1}
can be computed in the row streaming model using random
projections of the columns, subsampling the rows of the matrix proportional to their squared lengths \citep{drineas2006fast} or deterministically by using the
Frequent Directions algorithm \citep{Liberty16}. Our streaming algorithm is stated as Algorithm \ref{alg:FD}, and is very simple. In  Algorithm \ref{alg:FD}, any other sketch such as subsampling the rows of the matrix or using a random projection can also be used instead of Frequent Directions.

\begin{algorithm}
	\DontPrintSemicolon
	\SetAlgoLined
	\SetAlgoNoEnd
	\SetKwFunction{FDot}{Dot}
	\SetKwFunction{FUpdate}{Update}

	\newcommand\mycommfont[1]{\rmfamily{#1}}
	\SetCommentSty{mycommfont}
	\SetKwComment{Comment}{$\triangleright$ }{}

	\Input{Choice of $k$, sketch size $\ell$ for Frequent Directions \citep{Liberty16}}

	\FirstPass{}{
		Use Frequent Directions to compute a sketch $\tilde{\Ab}\in \mathbb{R}^{\ell \times d}$ \;
	}
	\SVD{}{
		Compute the top $k$ right singular vectors of $\tilde{\Ab}^T\tilde{\Ab}$
		}
	\SecondPass{As each row $a_{(i)}$ streams in,}{
		Use estimated right singular vectors to compute leverage scores and projection distances \;
	}

	\caption{Algorithm to approximate anomaly scores using Frequent Directions}
	\label{alg:FD}
\end{algorithm}

We state our results here, see
Section \ref{sec:point} for precise statements, proofs and general results for any sketches which satisfy the guarantee in Eq. \eqref{eq:guarantee1}.

\begin{theorem}\label{thm:FD}
  Assume that $\Ab$ is $(k, \Delta)$-separated. There exists
  $\ell=k^2\cdot \poly(\eps^{-1},\kappa_k, \Delta)$, such that the above algorithm  computes estimates $\td{T}^k(i)$ and $\td{L}^k(i)$ where
\begin{align*}
|T^k(i) - \td{T}^k(i)| &\leq \eps\sqnorm{\rowa{i}},\\
|L^k(i) - \td{L}^k(i)| &\leq \eps k \frac{\sqnorm{\rowa{i}}}{\|\Ab\|_F^2}.
\end{align*}
The algorithm uses memory $O(d\ell)$ and has running time $O(nd\ell)$.
\end{theorem}
The key is that while $\ell$ depends on $k$ and other
parameters, it is independent of $d$. In the setting where all these
parameters are constants independent of $d$, our memory requirement is
$O(d)$, improving on the trivial $O(d^2)$ bound.

Our approximations are additive rather than multiplicative.
But for anomaly detection, the candidate anomalies are ones
where $L^k(i)$ or $T^k(i)$ is large, and in this regime, we argue
below that our additive bounds also translate to good multiplicative
approximations. The additive error in computing $L^k(i)$ is about $\eps
k/n$ when all the rows have roughly equal norm. 
Note that the average rank-$k$ leverage score of all the rows of any
matrix with $n$ rows is $k/n$, hence a reasonable threshold on
$L^k(i)$ to regard a point as an anomaly is when $L^k(i) \gg k/n$, so
the guarantee for $L^k(i)$ in Theorem~\ref{thm:FD} preserves anomaly
scores up to a small multiplicative error for candidate anomalies, and
ensures that points which were not anomalies before are not mistakenly
classified as anomalies. For $T^k(i)$, the additive error for row $\rowa{i}$ is  $\eps
\sqnorm{\rowa{i}}$. Again, for points that are anomalies,
$T^k(i)$ is a constant fraction of $\sqnorm{\rowa{i}}$, so this
guarantee is meaningful.

Next we show  that substantial
savings are unlikely for any algorithm with strong pointwise
guarantees: there is an $\Omega(d)$ lower bound for any approximation
that lets you distinguish $L^k(i) =1$ from $L^k(i) = \eps$ for any
constant $\eps$. The precise statement and result appears in Section \ref{sec:lower_bound}. 
\begin{theorem}\label{thm:lower}
	Any streaming algorithm which takes a constant number of passes over the data and can compute a $0.1$ error additive approximation to  the
	rank-$k$ leverage scores or the rank-$k$ projection distances for all the rows of a matrix {must} use $\Omega(d)$ working space.
\end{theorem}

\paragraph{Average-case guarantees from columns space approximations.}

We derive smaller space algorithms, albeit with weaker guarantees
using sketches that give columns space approximations that satisfy
Equation \eqref{eq:guarantee2}. Even though the sketch gives column space
approximations our goal is still to compute the row anomaly scores, so
it not just a matter of working with the transpose. Many sketches are
known which approximate $\Ab\Ab^T$ and satisfy
Equation \eqref{eq:guarantee2}, for instance, a low-dimensional
projection by a random matrix $\Rb\in \mathbb{R}^{d   \times \ell}$
(e.g., each entry of $\Rb$ could be a scaled i.i.d. uniform $\{\pm
1\}$ random variable) satisfies Equation \eqref{eq:guarantee2} for
$\ell=O(k/\mu^2)$ \citep{koltchinskii2017}.

On first glance it is unclear how such a sketch should be useful: the matrix
$\td{\Ab}\td{\Ab^T}$ is an $n \times n$ matrix, and since $n \gg d$ this matrix is too expensive to
store. Our streaming algorithm avoids this problem by only
computing $\td{\Ab}^T\td{\Ab}$, which is an $\ell \times \ell$  matrix, and the larger matrix
$\td{\Ab}\td{\Ab}^T$ is only used for the analysis. Instantiated with
the sketch above, the resulting algorithm is simple to describe
(although the analysis is subtle): we pick a random matrix in $\Rb\in
\mathbb{R}^{d \times \ell}$ as above and return the anomaly scores for
the sketch $\td{\Ab} = \Ab\Rb$ instead. Doing this in a streaming
fashion using even the naive algorithm requires computing the small
covariance matrix $\td{\Ab}^T\td{\Ab}$, which is only $O(\ell^2)$ space.

But notice that we have not accounted for the space needed to store the $(d\times \ell)$
matrix $\Rb$. This is a subtle (but mainly theoretical) concern, which
can be addressed by using powerful results from the theory of
pseudorandomness \cite{vadhan2012pseudorandomness}. Constructions of
pseudorandom Johnson-Lindenstrauss matrices
\citep{cohen2015optimal, cohen2015dimensionality} imply that the
matrix $\Rb$ can be pseudorandom, meaning that it has a succinct
description using only $O(\log(d))$ bits, from which each entry
can be efficiently computed on the fly.

\begin{algorithm}
	\DontPrintSemicolon
	\SetAlgoLined
	\SetAlgoNoEnd
	\SetKwFunction{FDot}{Dot}
	\SetKwFunction{FUpdate}{Update}

	\newcommand\mycommfont[1]{\rmfamily{#1}}
	\SetCommentSty{mycommfont}
	\SetKwComment{Comment}{$\triangleright$ }{}

	\Input{Choice of $k$, random projection matrix $\Rb \in \mathbb{R}^{d\times \ell}$}
	\Init{}{Set covariance $\tilde{\Ab}^T\tilde{\Ab}\leftarrow 0$}
	\FirstPass{As each row $a_{(i)}$ streams in,}{
		Project by $\Rb$ to get $\Rb^T a_{(i)}$  \;
		Update covariance $\tilde{\Ab}^T\tilde{\Ab}\leftarrow \tilde{\Ab}^T\tilde{\Ab} + (\Rb^T a_{(i)})(\Rb^T a_{(i)})^T $
	}
	\SVD{}{
		Compute the top $k$ right singular vectors of $\tilde{\Ab}^T\tilde{\Ab}$
	}
	\SecondPass{As each row $a_{(i)}$ streams in,}{
		Project by $\Rb$ to get $\Rb^T a_{(i)}$  \;
		For each projected row, use the estimated right singular vectors to compute the leverage scores and projection distances \;
	}

	\caption{Algorithm to approximate anomaly scores using random projection}
	\label{alg:random}
\end{algorithm}


\begin{theorem}\label{thm:random_proj}
	For $\eps$ sufficiently small, there exists $\ell =
        k^3 \cdot \poly(\eps^{-1}, \Delta)$ such that
        the algorithm above produces estimates $\td{L}^k(i)$ and
        $\td{T}^k(i)$ in the second pass, such that with high probabilty,
	\begin{align*}
	\sum_{i=1}^{n}|T^k(i) - \td{T}^k(i)| &\leq \eps
        {\normfsq{\Ab}},\\
	\sum_{i=1}^{n}|L^k(i) - \td{L}^k(i)| &\leq \eps \sum_{i=1}^{n}L^k(i).
	\end{align*}
	The algorithm uses space $O(\ell^2+\log(d)\log(k))$ and has running time $O(nd\ell)$.
\end{theorem}

This gives an average case guarantee. 
We note that Theorem \ref{thm:random_proj} shows a new property of random
projections---that on average they can preserve leverage scores and
distances from the principal subspace, with the projection dimension
$\ell$ being only $\poly(k,\eps^{-1},\Delta)$, independent of both $n$
and $d$.


We can obtain similar guarantees as in Theorem \ref{thm:random_proj}
for other sketches which preserve the column space, such as sampling
the columns proportional to their squared lengths
\cite{drineas2006fast, magen2011low}, at the price of one extra pass.
Again the resulting algorithm is very simple: it  maintains a
carefully chosen $\ell \times \ell$ submatrix of the full $d \times d$
covariance matrix $\Ab^T\Ab$ where $\ell = O(k^3)$. 
We state the full algorithm in Section \ref{sec:subsample}.

\eat{

\subsection{Technical Contributions}

Our main technical contribution is to prove perturbation bounds for
various matrix operators that arise in the computation of various
anomaly scores. Perturbation bounds describe the effect of the addition
of a {\em small} noise matrix $\Nb$ to a base matrix $\Ab$ on various quantities and operators
associated with the matrix. The basic results here include
Weyl's inequality (c.f. \citet{horn1994} Theorem 3.3.16) and Wedin's
theorem~\citep{Wedin1972}, which respectively give such bounds for eigenvalues and
eigenvectors. We use these results to derive perturbation bounds on
projection operators that are used to compute anomaly
scores, these typically involve projecting onto the
top-$k$ principal subspace, and rescaling each co-ordinate by some
function of the corresponding singular values.

We derive our results by combining these perturbation bounds with
powerful results on matrix sketching, such as the Covariance
Estimation results of \citep{koltchinskii2017} and the Frequent
Directions sketch \citep{liberty2013simple,Liberty16}.  The sketches
can be viewed as computing a noisy version of the true covariance
matrices, hence our perturbation results can be applied.

With regards to our guarantees for random projections, a conceptual contribution of our work is to show that powerful results on covariance estimation from the statistics community
\citep{koltchinskii2017} can be leveraged to show improved guarantees for sketching matrices via random projections. Usual bounds for random projections rely on
properties such as the Johnson-Lindenstrauss property and subspace
embeddings, which require the projection dimension to depend on
the dimensions $d$ and $n$ of the input matrix or on its rank, whereas we show that using results on covariance estimation
\citep{koltchinskii2017} these guarantees can be satisfied with the projection dimension only depending on the rank $k$ of the  principal subspace where most of the variance of the data lies---which is significantly better as $k$ is often
small for real data. We believe this connection could have applications outside anomaly detection as well.

}


\section{Experimental Evaluation}\label{sec:experiments}
\vspace{-4pt}
The aim of our experiments is to test whether our algorithms give
comparable results to exact anomaly score computation based on full
SVD. So in our experiments, we take the results of SVD as the ground
truth and see how close our algorithms get to it. In particular, the
goal is to determine how large the parameter $\ell$ that determines
the size of the sketch needs to be to get close to the exact
scores. Our results suggest that for high dimensional
data sets, it is possible to get good approximations to the exact
anomaly scores even for fairly small values of $\ell$ (a small
multiple of $k$), hence our worst-case theoretical bounds (which
involve polynomials in $k$ and other parameters) are on the
pessimistic side.
%
\vspace{-18pt}\paragraph{Datasets: }
We ran experiments on three publicly available
datasets: p53 mutants \cite{p53}, Dorothea \cite{dorothea} and RCV1 \cite{rcv1}, all of which are available
from the UCI Machine Learning Repository, and are high dimensional ($d
> 5000$). The original RCV1 dataset contains 804414 rows, we took
every tenth element from it. The sizes of the datasets are listed in Table \ref{table:timing}.

\eat{

We ran experiments on the publicly available
datasets listed in Table \ref{tabl:data}, all of which are available
from the UCI Machine Learning Repository, and are high dimensional ($d
> 5000$). The original RCV1 dataset contains 804414 rows, we took
every tenth element from it.
	}
\eat{
\begin{table*}
	\caption{Datasets considered for evaluating our algorithms. }
	\label{tabl:data}
	\centering
  \begin{tabular}{ c  c  c  c }
    \toprule
    Dataset & Number of data points ($n$) & Dimension ($d$) & Results\\
    \midrule
    p53 Mutants \cite{p53}  & 16772 & 5409  &  Figure \ref{fig:p53}\\
    Dorothea \cite{dorothea}  & 1950 & 100000 & Figure \ref{fig:dorothea} \\
    RCV1 \cite{rcv1}    & 80442 & 47236   & Figure \ref{fig:rcv1}\\
    \bottomrule
  \end{tabular}
\end{table*}
}

\vspace{-8pt}
\paragraph{Ground Truth: }To establish the ground truth, there are two
parameters:  the dimension $k$ (typically between $10$ and $125$) and
a threshold $\eta$ (typically between $0.01$ and $0.1$). We compute
the anomaly scores for this $k$ using a full SVD, and then label the
$\eta$ fraction of points with the highest anomaly scores to be
outliers. $k$ is chosen by examining the explained variance of the
datatset as a function of $k$, and $\eta$ by examining the histogram
of the anomaly score.
\vspace{-8pt}
\paragraph{Our Algorithms: }We run Algorithm \ref{alg:FD} using random
column projections in place of Frequent Directions.\footnote{Since the
    existing implementation of Frequent Directions \cite{liberty-gh} does not seem to handle sparse matrices.}
The relevant parameter here is the projection dimension
$\ell$, which results in a sketch matrix of size $d \times \ell$.
We run Algorithm \ref{alg:random} with random row projections. If the
projection dimension is $\ell$, the resulting sketch size is
$O(\ell^2)$ for the covariance matrix. For a given $\ell$, the
time complexity of both algorithms is similar, however the size of the sketches
are very different: $O(d\ell)$ versus $O(\ell^2)$.
\vspace{-10pt}
\paragraph{Measuring accuracy: }
We ran experiments with a range of $\ell$s in the interval $(2k, 20k)$
for each dataset (hence the curves have different start/end
points). The algorithm is given just the
points (without labels or $\eta$) and computes anomaly scores for them. 
We then declare the points with the top $\eta'$ fraction of scores to be anomalies, and then compute the $F_1$
score (defined as the harmonic mean of the precision and the
recall). We choose $\eta'$ to maximize the $F_1$
score. This measures how well the proposed algorithms can approximate the exact anomaly scores. Note that in order to get both good precision and recall, $\eta'$ cannot be too far from $\eta$. We report the average $F_1$ score over $5$ runs.

For each dataset, we run both algorithms, approximate both the leverage and projection
scores, and try three different values of $k$. For each of these
settings, we run over roughly $10$ values for $\ell$.
The results are plotted in Figs. \ref{fig:p53}, \ref{fig:dorothea} and \ref{fig:rcv1}. Here are some takeaways from our experiments:
\vspace{-6pt}
\begin{itemize}
  \item Taking $\ell = C k$ with a fairly small $C \approx 10$ suffices to get F1 scores $>0.75$ in most settings.
  \item Algorithm \ref{alg:FD} generally outperforms Algorithm
    \ref{alg:random} for a given value of $\ell$. This should not be
    too surprising given that it uses much more memory, and is known
    to give pointwise rather than average case guarantees. However,
    Algorithm \ref{alg:random} does surprisingly well for an algorithm
    whose memory footprint is essentially independent of the input
    dimension $d$.
  \item The separation assumption (Assumption (1)) does hold to the
    extent that the spectrum is not degenerate, but not with a large
    gap. The algorithms seem fairly robust to this.
   \item The approximate low-rank assumption (Assumption (2)) seems to be
     important in practice. Our best results are for the p53 data set,
     where the top 10 components explain $87\%$ of the total
     variance. The worst results are for the RCV1 data
     set, where the top 100 and 200 components explain only
     $15\%$ and $25\%$ of the total variance respectively.
\end{itemize}
\vspace{-12pt}
\paragraph{Performance. }
While the main focus of this work is on the streaming model and memory
consumption, our algorithms offer considerable speedups even in the
offline/batch setting. Our timing experiments  were run using
Python/Jupyter notebook on a linux VM with 8 cores and 32 Gb of RAM,
the times reported are total CPU times  in seconds 
and are reported in
Table \ref{table:timing}. We focus on computing projection distances
using SVD (the baseline), Random Column Projection
(Algorithm \ref{alg:FD}) and Random Row Projection
(Algorithm \ref{alg:random}). All SVD computations use the
{\fontfamily{cmtt}\selectfont randomized$\_$svd} function from
{\fontfamily{cmtt}\selectfont scikit.learn}. The baseline computes
only the top $k$ singular values and vectors (not the entire SVD).
The results show consistent speedups between $2\times$ and $6\times$.  Which
algorithm is faster depends on which dimension of the input matrix is
larger.

\begin{table*}
	\centering
	\caption{Running times for computing rank-$k$ projection distance. Speedups between $2\times$ and $6\times$.}
	\label{table:timing}
	\begin{tabular}{ ccccccc }
		\toprule
		\multirow{2}{*}{Dataset}&
		\multirow{2}{*}{Size ($n \times d$)}&
		\multirow{2}{*}{$k$} &
		\multirow{2}{*}{$\ell$} &
		\multirow{2}{*}{SVD} & { Column } & { Row }\\
		& & & & & { Projection } & { Projection }\\
		\midrule
		p53 mutants & $16772 \times 5409$ & 20 & 200 & 29.2s & 6.88s & 7.5s \\
		Dorothea & $1950 \times  100000$ & 20 & 200 & 17.7s & 9.91s & 2.58s\\
		RCV1  & $80442 \times 47236$ & 50 & 500  & 39.6s & 17.5s & 20.8s\\
		\bottomrule
	\end{tabular}
\end{table*}

\begin{figure}
	\centering
	\includegraphics[width=0.45\textwidth]{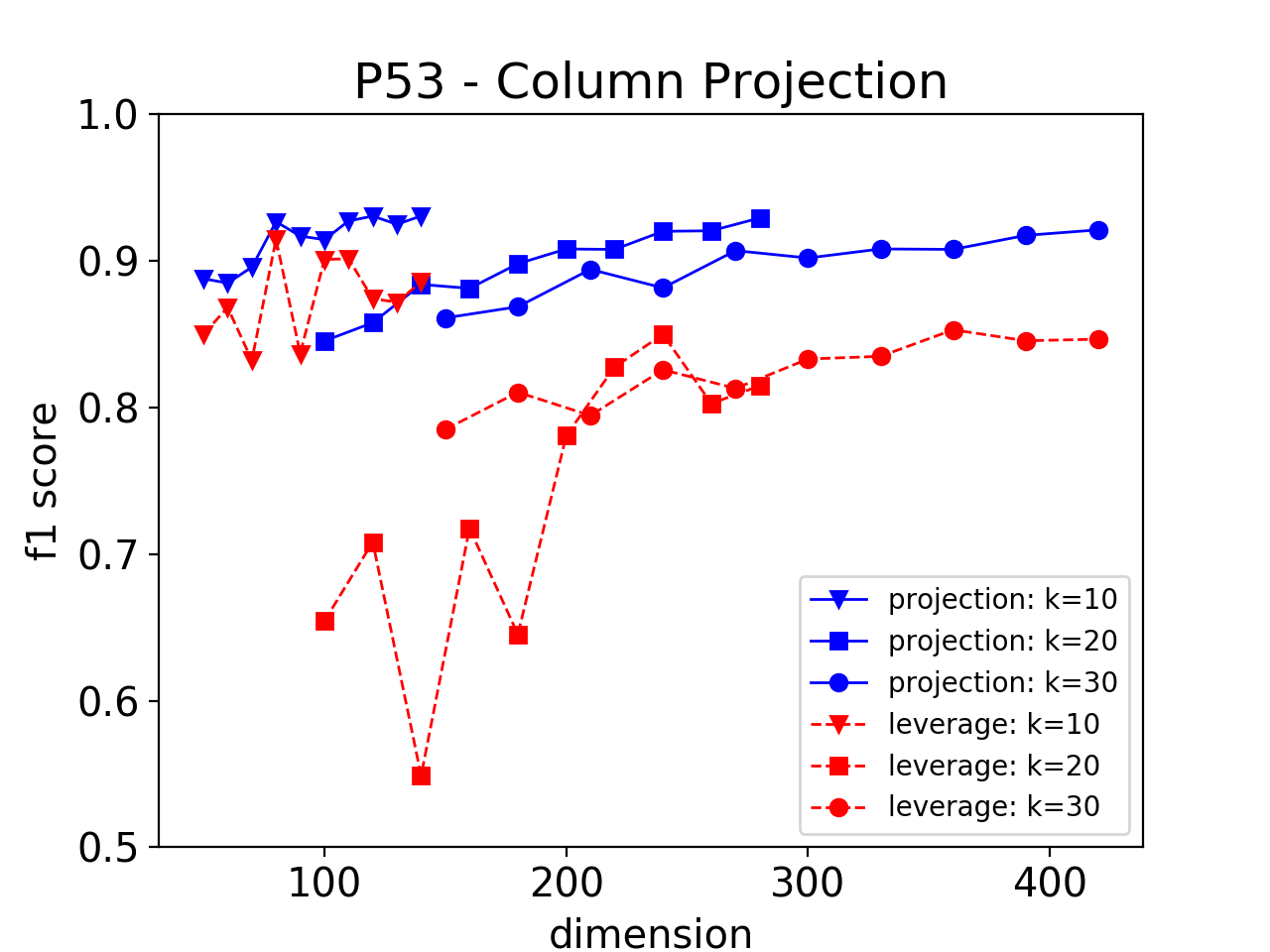}
	\includegraphics[width=0.45\textwidth]{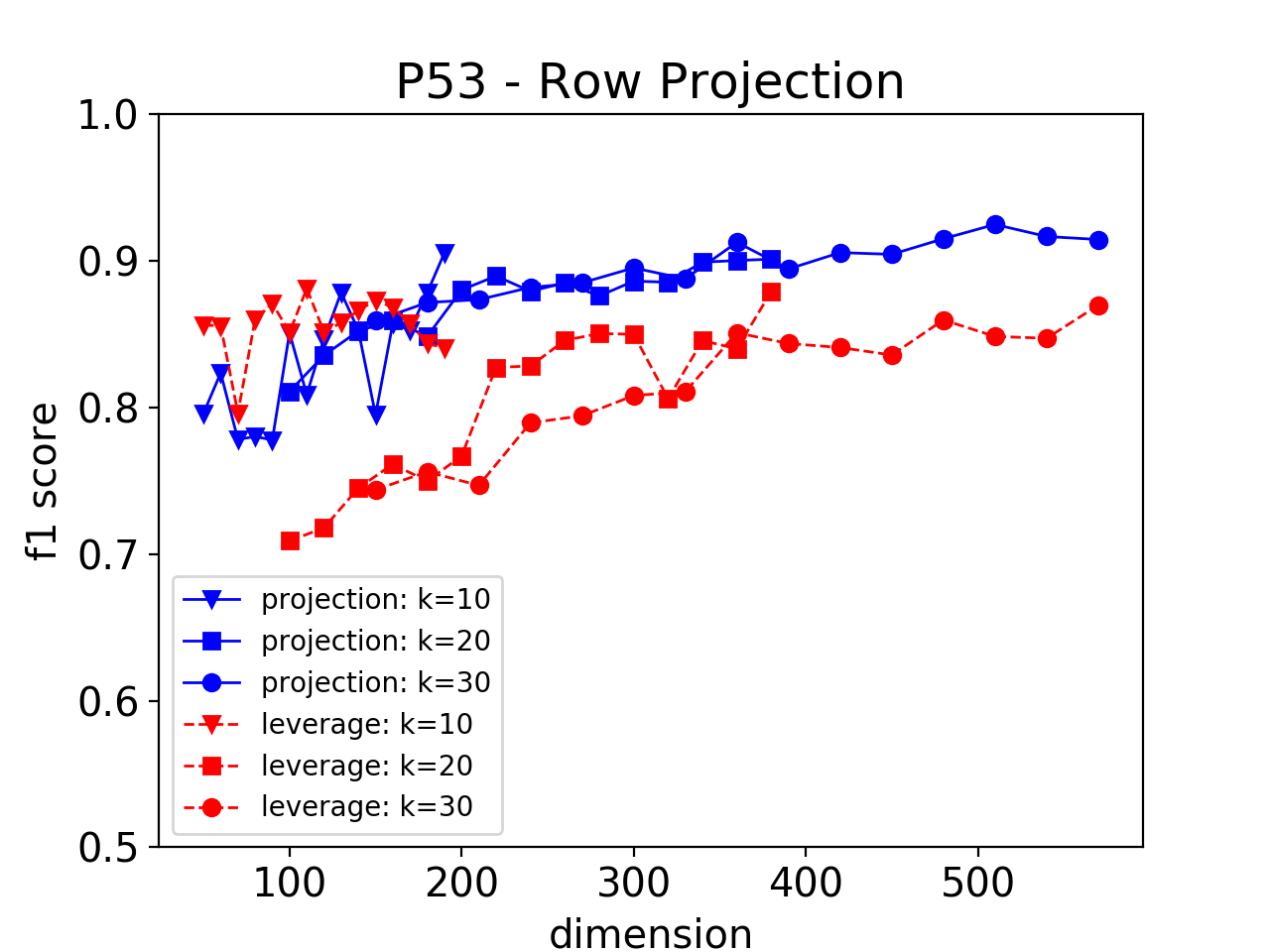}
	\caption{Results for P53 Mutants.
		We get $F_1$ score $>0.8$ with $>10\times$ space savings.}
	\label{fig:p53}
\end{figure}

\begin{figure}
	\centering
	\includegraphics[width=0.45\textwidth]{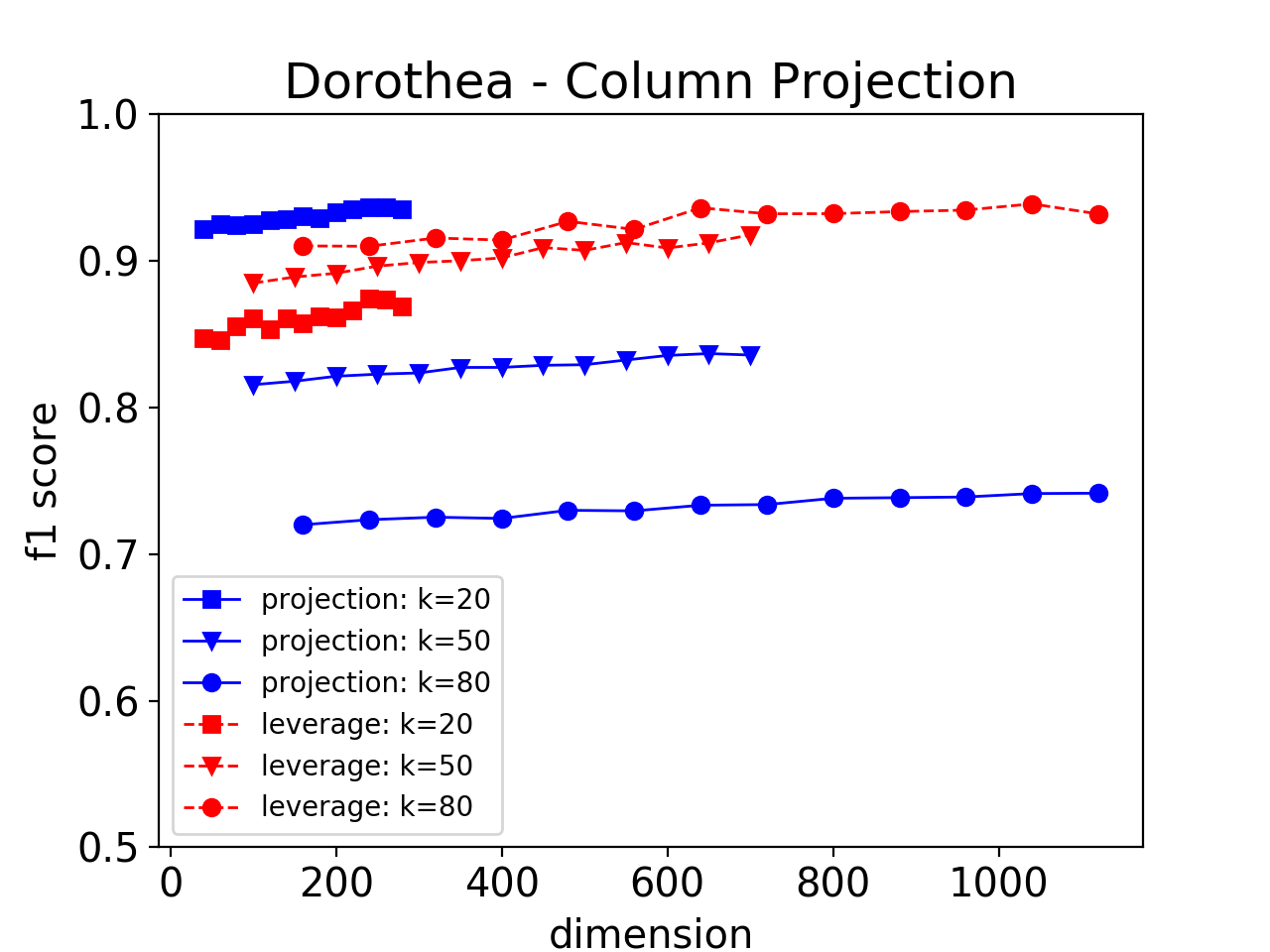}
	\includegraphics[width=0.45\textwidth]{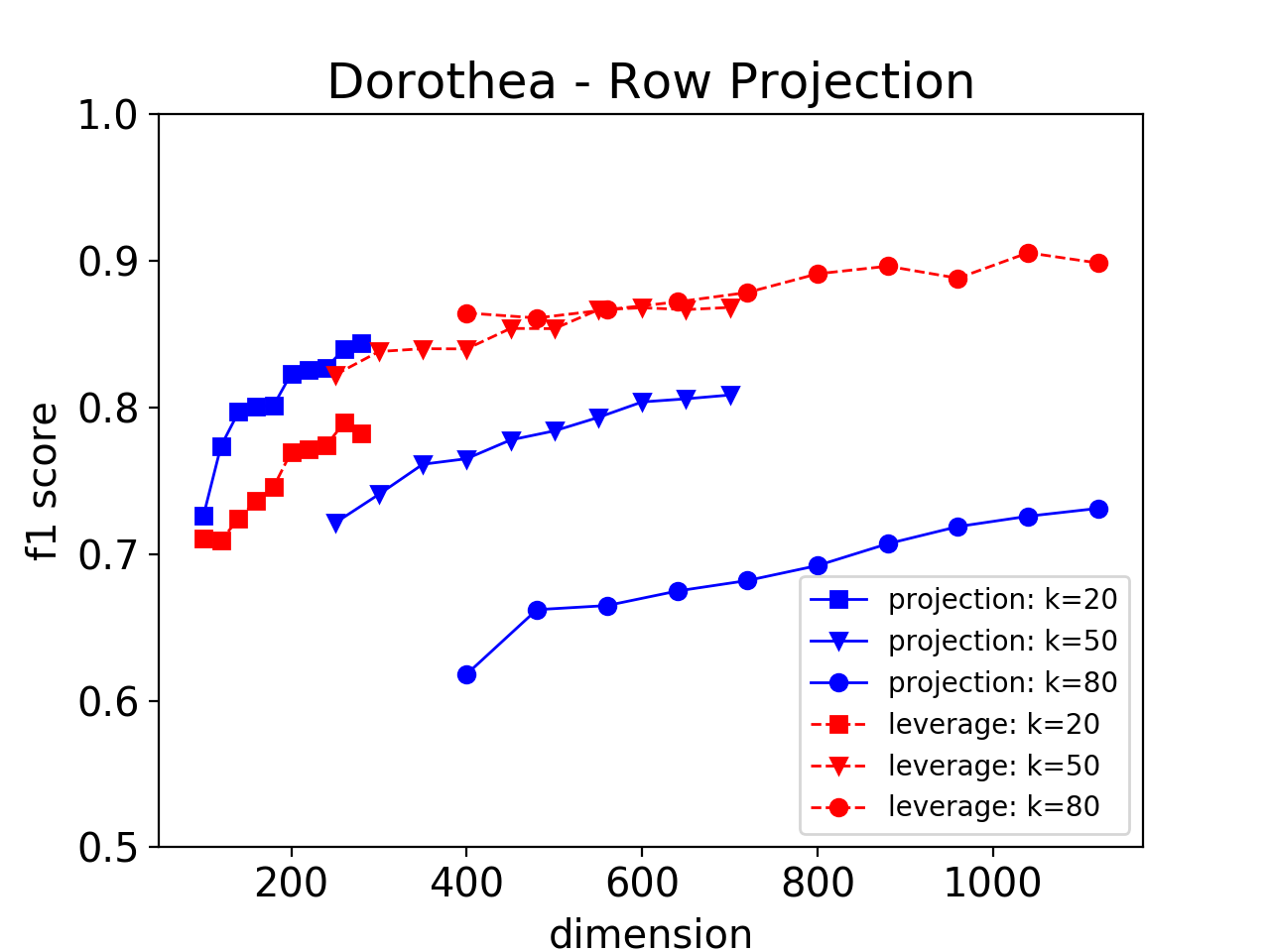}
	\caption{Results for the Dorothea dataset. Column projections give more accurate approximations, but they use more space. }
	\label{fig:dorothea}
\end{figure}

\begin{figure}
	\centering
	\includegraphics[width=0.45\textwidth]{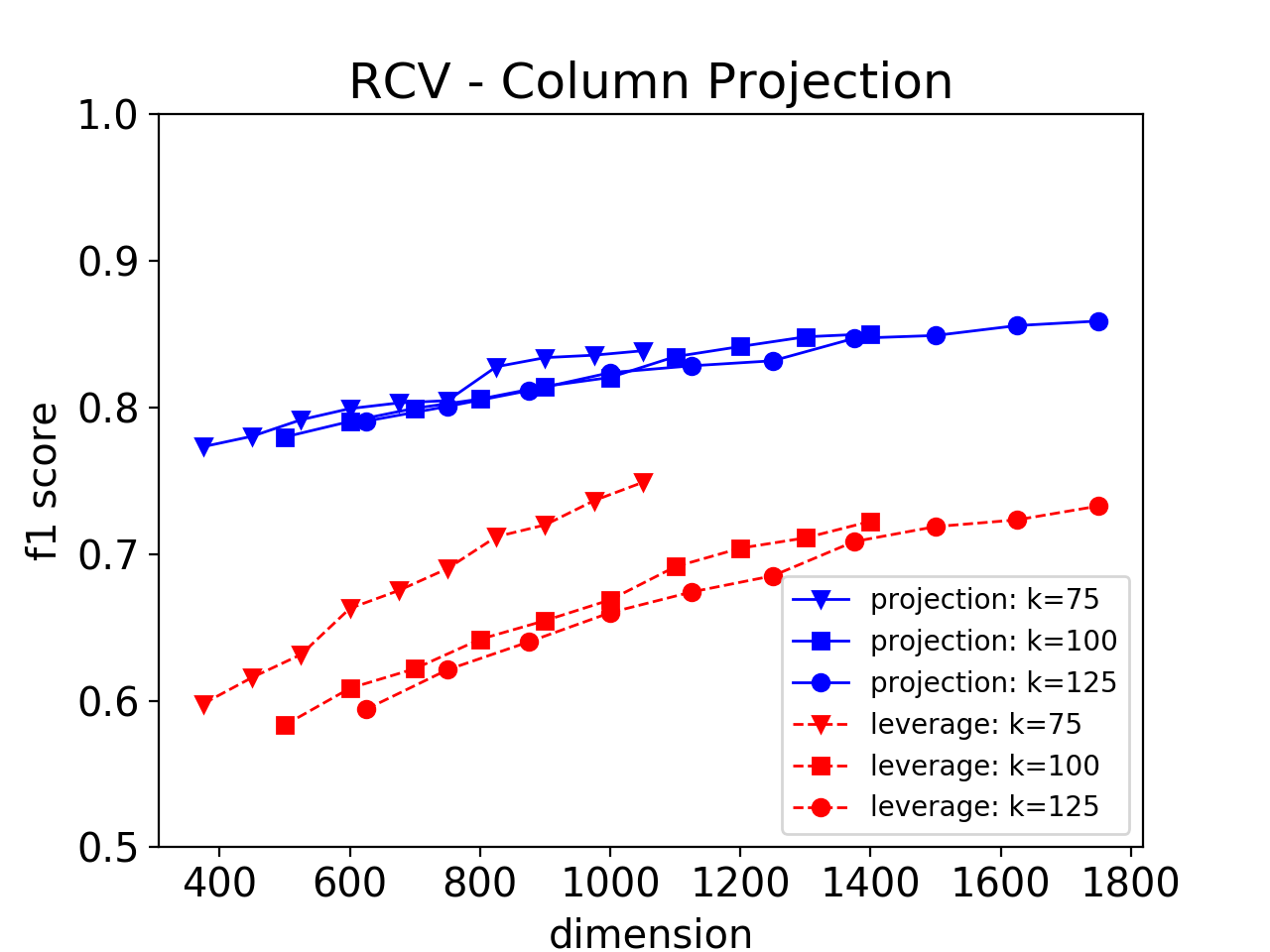}
	\includegraphics[width=0.45\textwidth]{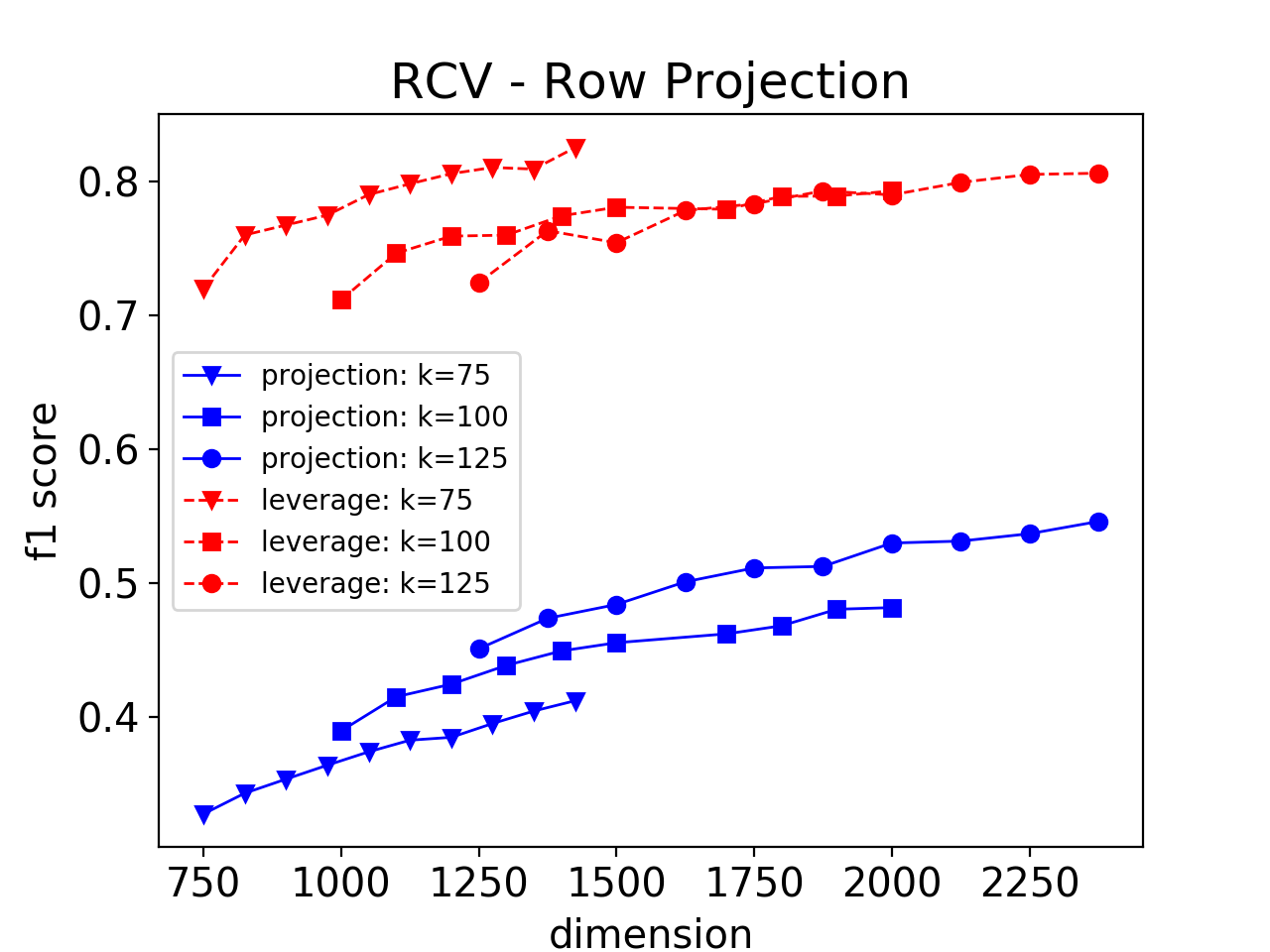}
	\caption{Results for the RCV1 dataset. Our results here are worse than for the other datasets, we hypothesize this is due to this data having less pronounced low-rank structure.}
	\label{fig:rcv1}
\end{figure}

\eat{
\paragraph{Conclusions. }We have shown that techniques from sketching can be used to derive
algorithms for computing subspace-based anomaly scores which
provably approximate the true SVD-based scores but at a lower cost in terms of
time and memory. Another resource that this optimizes is communication,
which is important in the IoT regime where the rows
corresponds to data collected by different sensors, and the
communication cost of transmitting the data from the sensors to a
central server is a bottleneck, and where anomaly detection is an important problem.
\eat{In such a scenario, the sensors can
compress the  data by subsampling or random projection before sending
them to the server. The server would reply with (the SVD of) the
aggregate covariance, which the sensors could then use to compute the
anomaly scores locally.}
}

\vspace{-8pt}

\section{Related Work}\label{sec:related}

\vspace{-6pt}
In most anomaly detection settings, labels are  hard to come by and
unsupervised learning methods are preferred: the algorithm needs to learn what
the bulk of the data looks like and then detect any deviations from this. Subspace based scores are well-suited to this, but various other anomaly scores have also been proposed such as those based on approximating the density of the data \citep{breunig2000lof,schneider2016expected} and attribute-wise analysis \citep{liu2008isolation}, we refer to surveys on anomaly detection for an overview \citep{Aggarwal2013,ChandolaBK09}. 

Leverage scores have found numerous
applications in numerical linear algebra, and hence there has been
significant interest in improving the time complexity of computing
them. For the problem of approximating the (full) leverage scores ($L(i)$ in Eq. \eqref{eq:full_lev}, note that we are concerned with the rank-$k$ leverage scores $L^k(i)$), \citet{clarkson2013low} and \citet{drineas2012fast} use sparse subspace embeddings and Fast Johnson Lindenstrauss Transforms (FJLT \citep{ailon2009fast}) to compute
the leverage scores using $O(nd)$  time instead of the $O(nd^2)$ time
required by the baseline---but these still need $O(d^2)$
memory. With respect to projection distance, the closest work to ours
is \citet{Huang15} which uses Frequent Directions to approximate
projection distances in $O(kd)$ space. In contrast to these approaches, our results hold both for rank-$k$ leverage scores and projection distances, for any matrix sketching algorithm---not just FJLT or Frequent Directions---and our space requirement can be as small as $\log(d)$ for average case guarantees. However, \citet{clarkson2013low} and \citet{drineas2012fast} give multiplicative guarantees for approximating leverage scores while our guarantees for rank-$k$ leverage scores are additive, but are nevertheless sufficient for the task of detecting anomalies.


\section{Matrix Perturbation Bounds}
\label{sec:prelims}

In this section, we will establish projection bounds for various
operators needed for computing outlier scores. We first set up some
notation and state some results we need.

\subsection{Preliminaries}

We work with the following setup throughout this section. Let $\Ab \in
\R^{n \times d} = \Ub\Sigmab\Vb^T$ where $\Sigmab = (\sigma_1, \ldots,
\sigma_d)$. Assume that $\Ab$ is $(k, \Delta)$-separated as in Assumption
1 from \ref{sec:setup}. We use $\sr(\Ab) = \normfsq{\Ab}/\sigma_1^2$ to denote
the stable rank of $\Ab$, and $\kappa_k = \sigma_1^2/\sigma_k^2$ for the
condition number of $\Ab_k$.

Let $\td{\Ab} \in \R^{n \times   \ell}$ be a  sketch/noisy version of $\Ab$ satisfying
\begin{align}
\label{eq:setup-0}
\norm{\Ab\Ab^T - \td{\Ab}\td{\Ab}^T} \leq \mu\sigma_1^2.
\end{align}
and let $\td{\Ab}= \td{\Ub}\td{\Sigmab}\td{\Vb}^T$ denote its SVD.
While we did not assume $\td{\Ab}$ is $(k, O(\Delta))$-separated, it
will follow from Weyl's inequality for $\mu$ sufficiently small
compared to $\Delta$. It helps to think of $\Delta$ as property of the
input $\Ab$, and $\mu$ as an accuracy parameter that we control.

In this section we prove perturbation  bounds for the following three operators derived from $\td \Ab$, showing their closeness to those derived from $\Ab$:
\begin{enumerate}
  \item \label{item:1}$\Ub_k\Ub_k^T$: projects onto the principal
    $k$-dimensional column subspace (Lemma \ref{lem:wed}).
  \item \label{item:2}$\Ub_k\Sigmab_k^2\Ub_k^T$: projects onto the principal
    $k$-dimensional column subspace, and  scale co-ordinate $i$ by
    $\sigma_i$ (Theorem \ref{thm:sigma1}).
  \item \label{item:3}$\Ub_k\Sigmab^{-2}_k\Ub_k^T$: projects onto the principal
    $k$-dimensional column subspace, and  scale co-ordinate $i$ by
    $1/\sigma_i$ (Theorem \ref{thm:sigma2}).
\end{enumerate}
To do so, we will extensively use two classical results about matrix
perturbation: Weyl's inequality (c.f. \citet{horn1994} Theorem 3.3.16)
and Wedin's theorem~\citep{Wedin1972}, which respectively quantify how
the eigenvalues and eigenvectors of a matrix change under
perturbations.

\begin{lemma}\label{lem:weyl}
	{\bf(Weyl's inequality)} Let $\Cb,\Db = \Cb + \Nb \in\mathbb{R}^{n
    \times d}$. Then for all $i\le \min(n,d)$,
	\[
	|\sigma_i(\Cb)-\sigma_i(\Db)|\le \norm{\Nb}.
	\]
\end{lemma}

Wedin's theorem requires a sufficiently large separation in the
eigenvalue spectrum for the bound to be meaningful.

\begin{lemma}\label{thm:wedin}
  {\bf(Wedin's theorem)} Let $\Cb,\Db = \Cb + \Nb \in \mathbb{R}^{n
    \times d}$. Let $\Pb_{C}$ and
  $\Pb_{D}$ respectively denote the projection matrix onto the space spanned by the top $k$ singular
  vectors of $\Cb$ and $\Db$. Then,
	\[
	\norm{\Pb_D-\Pb_C\Pb_D}\le \frac{\norm{\Nb}}{\sigma_k(\Cb) - \sigma_{k+1}(\Cb) -\norm{\Nb}}.
	\]
\end{lemma}

\subsection{Matrix Perturbation Bounds}

We now use these results to derive the perturbation
bounds enumerated above. The first bound is a direct consequence of Wedin's theorem.

\begin{lemma}\label{lem:wed}
  If $\Ab = \Ub\Sigma\Vb^T$ is $(k, \Delta)$-separated, $\td{\Ab}= \td{\Ub}\td{\Sigmab}\td{\Vb}^T$ satisfies \eqref{eq:setup-0} with $\mu \leq \Delta/6$, then
  \[ \norm{\Ub_k\Ub_k^T - \tilde{\Ub}_k\tilde{\Ub}_k^T} \le  2\sqrt{\frac{\mu}{\Delta}}. \]
\end{lemma}
\begin{proof}
	Let $\Pb = \Ub_k\Ub_k^T$ and $\td{\Pb} =
        \td{\Ub}_k\td{\Ub}_k^T$. Since we have $\Ab\Ab^T =
        \Ub\Sigma^2\Ub^T$, $\Pb$ ($\td{\Pb}$) is the projection
        operator onto the column space of $\Ab_k$ (and similarly for $\td{\Pb}$ and $\td{\Ab}_k$).
	Since $\Pb^T = \Pb$ and $\Pb\Pb = \Pb$ for any orthogonal projection matrix, we can write,
	\begin{align}
	\norm{\Pb -\td{\Pb}}^2 & = \norm{(\Pb -\td{\Pb})^2} =
	\norm{\Pb\Pb -\Pb\td{\Pb} -\td{\Pb}\Pb +\td{\Pb}\td{\Pb}} =
	\norm{\Pb-\Pb\td{\Pb}+\td{\Pb}-\td{\Pb}\Pb}\notag \\
	& \le \norm{\td{\Pb} -\Pb\td{\Pb}}+\norm{\Pb
		-\td{\Pb}\Pb}\label{eq:proj}.
	\end{align}

        Since $\Ab$ is $(k, \Delta)$-separated,
        \begin{align}
          \label{eq:setup-1}
          \sigma_k(\Ab\Ab^t) - \sigma_{k+1}(\Ab\Ab^t) = \sigma_k^2 -
          \sigma_{k+1}(\Ab)^2 = \Delta.
        \end{align}
        So applying Wedin's theorem to $\Ab\Ab^T$,
	\begin{align}
	\label{eq:wedin-1}
	\norm{\td{\Pb} -\Pb\td{\Pb}} \leq \frac{\mu}{\Delta - \mu}.
	\end{align}

	We next show that the spectrum of $\td{\Ab}$ also has a gap at
	$k$. Using Weyl's inequality,
	\begin{align*}
	\td{\sigma}_k^2 & \ge \sigma^2_k-\mu \sigma_1^2\text{ and } \
	\td{\sigma}_{k+1}^2 \le \sigma^2_{k+1} +\mu \sigma_1^2
        \end{align*}
        Hence using Equation \eqref{eq:setup-1}
        \begin{align*}
        \td{\sigma}_k^2 - \td{\sigma}_{k+1}^2 & \geq  \sigma^2_{k+1} -
	\sigma^2_k - 2\mu \sigma_1^2 \geq (\Delta - 2\mu)\sigma_1^2.
	\end{align*}
	So we apply Wedin's theorem to $\td{\Ab}$ to get
	\begin{align}
	\label{eq:wedin-2}
	\norm{\Pb - \td{\Pb}\Pb} \leq \frac{\mu}{\Delta - 3 \mu}.
	\end{align}

	Plugging \eqref{eq:wedin-1} and \eqref{eq:wedin-2}
	into \eqref{eq:proj},
	\begin{align*}
	\norm{\Pb -\td{\Pb}}^2 \le \frac{\mu}{\Delta - \mu} +
	\frac{\mu}{\Delta - 3\mu} \leq \frac{2\mu}{\Delta - 3\mu} \leq \frac{4\mu}{\Delta}
	\end{align*}
	where the last inequality is becuase $\Delta - 3\mu \geq
	\Delta/2$ since we assumed $\Delta \geq 6\mu$. The claim follows
	by taking square roots on both sides.
\end{proof}

We now move to proving bounds for items \eqref{item:2} and \eqref{item:3}, which
are given by Theorems \ref{thm:sigma1} and \ref{thm:sigma2}
respectively.

\begin{restatable}{theorem}{sigmaone}
	\label{thm:sigma1}
	Let $\mu \leq \min(\Delta^3k^2, 1/(20 k))$.
	\[ \norm{\Ub_{k}\Sigmab_k^2\Ub^T_{k} - \td{\Ub}_{
			k}\Sigmab_k^2\td{\Ub}^T_{k}} \leq 8\sigma_1^2(\mu k)^{1/3}.\]
\end{restatable} 
The full proof of the theorem is quite technical and is stated in Section \ref{sec:prelim_full}. Here we prove a special case that captures the main idea. The simplifying assumption we make is that the values of the diagonal matrix in the operator are distinct and well separated. In the full proof we decompose $\Sigmab_k$ to a well separated component and a small residual component. 

\begin{definition}
$\Lambdab = \diag(\lambda_1, \ldots \lambda_k)$ is \emph {monotone and $\delta$-well separated} if \begin{itemize}
	\item $\lambda_{i+1} \geq \lambda_i$ for $1 \leq i < k$.
	\item The $\lambda$s could be partitioned to buckets so that all values in the same buckets are equal, and values across buckets are well separated. Formally, $1\ldots k$ are partitioned to $m$ buckets $B_1,\ldots B_m$ so that if  $i,j \in B_\ell$ then $\lambda_i = \lambda_j$. Yet, if $i \in B_\ell$ and $j \in B_{\ell + 1}$ then $\lambda_i - \lambda_j > \delta \lambda_1$.
\end{itemize}  
\end{definition}
The idea is to show $\Sigmab = \Lambdab + \Omegab$ where $\Lambdab$ is monotone and well separated and $\Omegab$ has small norm.  The next two lemmas handle these two components. We first state a simple lemma which handles the case where $\Omegab$ has small norm.
\begin{lemma}\label{lem:diag1}
	For any diagonal matrix $\Omegab$,
	\[ \norm{\Ub_{k}\Omegab\Ub^T_{k} - \td{\Ub}_{k}\Omegab\td{\Ub}^T_{ k}} \leq 2\norm{\Omegab}.\]
\end{lemma}
\begin{proof}
	By the triangle inequality, $\norm{\Ub_{k}\Omegab\Ub^T_{k} -
		\td{\Ub}_{k}\Omegab\td{\Ub}^T_{ k}} \leq
	\norm{\Ub_{k}\Omegab\Ub^T_{k}} +
	\norm{\td{\Ub}_{k}\Omegab\td{\Ub}^T_{ k}}$. The bound follows
	as $\Ub_{k}$ and $\td{\Ub}$ are orthonormal matrices.
\end{proof}
We next state the result for the case where $\Lambdab$ is monotone and well separated. In order to prove this result, we need the following direct corollary of Lemma \ref{lem:wed}:

\begin{lemma}
	\label{lem:non_dec1}
	For all $j \le m$,
	\begin{align}
	\norm{\Ub_{b_j}\Ub_{b_j}^T-\td{\Ub}_{b_j}\td{\Ub}_{b_j}^T}\le
	\sqrt{\frac{2\mu\sigma_1^2}{\delta\sigma_1^2-3\mu\sigma_1^2}}\le
	2\sqrt{\frac{\mu}{\delta}}.\label{eq:non_dec1}
	\end{align}
\end{lemma}

Using this, we now proceed as follows.

\begin{lemma}
	\label{lem:diag2_simple}
	Let $6\mu \leq \delta \leq \Delta$. Let $\Lambdab =
	\diag(\lambda_1,\ldots\lambda_k)$ be a monotone and $\delta \sigma_1^2$ well separated diagonal  matrix.  Then
	\[ \norm{\Ub_k\Lambdab\Ub^T_{k} - \td{\Ub}_k\Lambdab\td{\Ub}^T_{k}}
	\leq 2\norm{\Lambdab}\sqrt{\frac{\mu}{\delta}}.\]
\end{lemma}
\begin{proof}
	We denote by ${b_j}$ the largest index that falls in bucket $B_j$.  
	Let us set $\lambda_{B_{m+1}} =  0$ for convenience. Since
	\begin{align*}
	\Ub_{b_j}\Ub_{b_j}^T  = \sum_{i =
		1}^{b_j}\colof{u}{i}{\colof{u}{i}}^T,
	\end{align*}
	we can write,
	\begin{align*}
	\Ub_{k} \Lambdab \Ub_{k}^T  =  \sum_{j=1}^m(\lambda_{b_j}
	- \lambda_{b_{j+1}})\Ub_{b_j}\Ub_{b_j}^T
	\end{align*}
	and similarly for $\td{\Ub}_{k} \Lambdab\td{\Ub}_{k}^T$. So we can
	write
	\begin{align*}
	\Ub_{k} \Lambdab \Ub_{k}^T -  \td{\Ub}_{k} \Lambdab\td{\Ub}_{k}^T
	= \sum_{j=1}^{m}(\lambda_{b_j}-\lambda_{b_{j+1}})(\Ub_{b_j}\Ub_{b_j}^T-\td{\Ub}_{b_j}\td{\Ub}_{b_j}^T).
	\end{align*}
	Therefore by using Lemma \ref{lem:non_dec1}, 
	\begin{align*}
	\norm{ \Ub_{k} \Lambdab \Ub_{k}^T -  \td{\Ub}_{k} \Lambdab\td{\Ub}_{k}^T}
	=
	\sum_{j=1}^{m}|\lambda_{b_j}-\lambda_{b_{j+1}}|\norm{(\Ub_{b_j}\Ub_{b_j}^T-\td{\Ub}_{b_j}\td{\Ub}_{b_j}^T)}
	\leq 2\sqrt\frac{\mu}{\delta} \sum_{j=1}^m|\lambda_{b_j}-\lambda_{b_{j+1}}|,
	\end{align*}
	where the second inequality is by the triangle inequality and by applying Lemma~\ref{thm:sigma1}.
	Thus proving that $\sum_{j=1}^m|\lambda_{b_j}-\lambda_{b_{j+1}}|
	\leq  \norm{\Lambdab}$ would imply the claim. Indeed
	
	\[ \sum_{j=1}^k|\lambda_{b_j}-\lambda_{b_{j+1}}| =
	\sum_{j=1}^k(\lambda_{b_j}-\lambda_{b_{j+1}}) =  \lambda_{b_1} -
	\lambda_{b_{k+1}} \le \norm{\Lambdab}.\]
\end{proof}

Note that though Lemma \ref{lem:diag2_simple} assumes that the eigenvalues in each bucket are equal, the key step where we apply Wedin's theorem (Lemma~\ref{lem:wed}) only uses the fact that there is some separation between the eigenvalues in different buckets. Lemma \ref{lem:diag2} in the appendix does this generalization by relaxing the assumption that the eigenvalues in each bucket are equal. The final proof of Theorem~\ref{thm:sigma1} works by splitting the eigenvalues into different buckets or intervals such that all the eigenvalues in the same interval have small separation, and the eigenvalues in different intervals have large separation. We then use Lemma \ref{lem:diag2} and Lemma \ref{lem:diag1} along with the triangle inequality to bound the perturbation due to the well-separated part and the residual part respectively.

The bound corresponding to \eqref{item:3} is given in following theorem:
\begin{restatable}{theorem}{sigmatwo}
	\label{thm:sigma2}
	Let $\kappa_k$ denote the condition
	number $\kappa_k =\sigma_1^2/\sigma_k^2$.
	Let $\mu \leq \min(\Delta^3(k\kappa_k)^2, 1/(20 k \kappa_k))$. Then,
	\[ \norm{\Ub_{k}\Sigmab_k^{-2}\Ub^T_{k} - \td{\Ub}_{
			k}\Sigmab_k^{-2}\td{\Ub}^T_{k}} \leq
	\frac{8}{\sigma_k^2}(\mu k \kappa_k)^{1/3}. \]
\end{restatable}

The proof uses similar ideas to those of Theorem~\ref{thm:sigma1} and is deferred to Section \ref{sec:prelim_full}.

\section{Pointwise guarantees for Anomaly Detection}\label{sec:point}

Let the input $\Ab \in \R^{n \times d}$ have SVD $\Ub \Sigmab \Vb^T$ be its SVD.
Write $\rowa{i}$ in the basis of right singular vectors as $\rowa{i} =
\sum_{j = 1}^d\alpha_jv^{(j)}$. Recall that we defined its rank-$k$
leverage score and projection distance respectively as
\begin{align}
L^k(i) & = \sum_{j = 1}^k \frac{\alpha_j^2}{\sigma_j^2} = \sqnorm{\Sigmab_k^{-1}\Vb^T_k\rowa{i}},\label{eq:l-onl}\\
T^k(i) & = \sum_{j = k+1}^d \alpha_j^2 = \rowa{i}^T(\Ib -\Vb_k\Vb^T_k)\rowa{i}.\label{eq:t-onl}
\end{align}

To approximate these scores, it is natural to use a
row-space approximation, or rather a
sketch $\td{\Ab} \in R^{\ell \times d}$ that approximates the
covariance matrix $\Ab^T\Ab$ as below:
\begin{align}
  \label{eq:cov}
  \norm{\Ab^T\Ab - \td{\Ab}^T\td{\Ab}} \leq \mu \sigma_1^2.
\end{align}
Given such a sketch, our approximation is the following: compute $\td{\Ab} =
\td{\Ub}\td{\Sigmab}\td{\Vb}^t$. The estimates for $L^k(i)$ and $T^k(i)$ respectively are
\begin{align}
  \td{L}^k(i) &  = \sqnorm{\td{\Sigmab}_k^{-1}\td{\Vb}^T_k\rowa{i}},\label{eq:lt-onl}\\
  \td{T}^k(i) &  = \rowa{i}^T(\Ib -\tilde\Vb_k\tilde\Vb^T_k)\rowa{i}.\label{eq:tt-onl}
\end{align}

Given the sketch $\td{\Ab}$, we expect that $\td{\Vb}_k$ is a
good approximation to the row space spanned by $\Vb_k$, since the
covariance matrices of the rows are close. In contrast the
columns spaces of $\Ab$ and $\td{\Ab}$ are hard to compare since they
lie in $\R^n$ and $\R^\ell$ respectively. The closeness of the row
spaces follows from the results from Section \ref{sec:prelims} but
applied to $\Ab^T$ rather than $\Ab$ itself. The results there require that  $\norm{\Ab\Ab^T -
  \td{\Ab}\td{\Ab}^T}$ is small, and Equation \eqref{eq:cov} implies that
this assumption holds for $\Ab^T$.

We first state our approximation guarantees for $T^k$.

\begin{theorem}
  \label{thm:t-online}
  Assume that $\Ab$ is $(k, \Delta)$-separated. Let $\eps <1/3$ and let $\td{\Ab}$ satisfy
  Equation \eqref{eq:cov} for $\mu = \eps^2\Delta$. Then for every $i$,
  \begin{align*}
    |T^k(i) - \td{T}^k(i)| \leq \eps\sqnorm{\rowa{i}}.
  \end{align*}
\end{theorem}
\begin{proof}
  We have
  \begin{align*}
    |T^k(i) - \td{T}^k(i)| & = |\rowa{i}^T(\Ib -\Vb_k\Vb^T_k)\rowa{i} -
    \rowa{i}^T(\Ib - \tilde{\Vb}_k\tilde{\Vb}^T_k)\rowa{i}| \\
    & = |\rowa{i}(\tilde{\Vb}_k\tilde{\Vb}^T_k -
    \Vb_k\Vb_k^T)\rowa{i}| \\
    & \leq \norm{\tilde{\Vb}_k\tilde{\Vb}^T_k - \Vb_k\Vb_k^T}  \sqnorm{\rowa{i}}\\
    & \leq \eps \sqnorm{\rowa{i}}
  \end{align*}
  where in the last line we use Lemma \ref{lem:wed}, applied to the
  projection onto columns of $\Ab^T$, which are the rows of $\Ab$. The condition
  $\mu < \Delta/6$ holds since $\mu = \eps^2\Delta$ for $\eps < 1/3$.
\end{proof}

How meaningful is the above additive approximation guarantee? For each
row, the additive error is $\eps\sqnorm{\rowa{i}}$. It might be
that $T^k(i) \ll \eps\sqnorm{\rowa{i}}$ which happens when the row is
almost entirely within the principal subspace. But in this case,
the points are not anomalies, and we have $\td{T}^k(i) \leq 2\eps
\sqnorm{\rowa{i}}$, so these points will not seem anomalous from the approximation either.
The interesting case is when $T^k(i) \geq \beta \sqnorm{\rowa{i}}$ for
some constant $\beta$ (say $1/2$). For such points, we have
$\td{T}^k(i) \in (\beta \pm \epsilon)\sqnorm{\rowa{i}}$, so we
indeed get multiplicative approximations.

Next we give our approximation guarantee for $L^k$, which relies on the perturbation bound in Theorem \ref{thm:sigma2}.

\begin{theorem}
  \label{thm:l-online}
  Assume that $\Ab$ is $(k, \Delta)$-separated. Let $\td{\Ab}$ be as in Equation \eqref{eq:cov}. Let
  \begin{align*}
    \eps \leq \min\Big(\kappa_k\Delta, \frac{1}{k}\Big)\sr(\Ab)\kappa_k,\quad \mu = \frac{\eps^3k^2}{10^3 \sr(A)^3\kappa_k^4}.
  \end{align*}
  Then for every $i$,
  \begin{align}
    |L^k(i) - \td{L}^k(i)| \leq \eps k
    \frac{\sqnorm{\rowa{i}}}{\|\Ab\|_F^2}. \label{eq:l-guarantee}
  \end{align}
\end{theorem}
\begin{proof}
  Since  $L^k(i)  = \sqnorm{\Sigmab_k^{-1}\Vb^T_k\rowa{i}}$,
  $\td{L}^k(i)  = \sqnorm{\td{\Sigmab}_k^{-1}\td{\Vb}^T_k\rowa{i}}$,
  it will suffice to show that
      \begin{align}
      \norm{\Vb_k\Sigmab_k^{-2}\Vb_k^T - \tilde{\Vb_k}\tilde{\Sigmab}^{-2}\tilde{\Vb}_k^T}
      \leq \frac{\eps k}{\|\Ab\|_F^2}\label{eq:mhnbs}.
      \end{align}
  To prove inequality \eqref{eq:mhnbs}, we will bound the LHS as
  \begin{align}
    \norm{\Vb_k\Sigmab^{-2}\Vb_k^T -
      \tilde{\Vb}_k\tilde{\Sigmab}^{-2}\tilde{\Vb}_k^T} \leq
    \norm{\Vb_k\Sigmab_k^{-2}\Vb_k^T - \tilde{\Vb}_k^T\Sigmab_k^{-2}\tilde{\Vb}_k^T} +
    \norm{\tilde{\Vb}_k(\Sigmab^{-2} - \tilde{\Sigmab}^{-2})\tilde{\Vb}_k^T}\label{eq:mhnbs-0}
  \end{align}
  For the first term, we apply Theorem \ref{thm:sigma2} to $\Ab^T$ to get
  \begin{align}
    \norm{\Vb_k\Sigmab^{-2}\Vb_k^T -
      \tilde{\Vb}_k^T\Sigmab_k^{-2}\tilde{\Vb}_k^T} \leq
    \frac{8}{\sigma_k^2}(\mu k \kappa_k)^{1/3}. \label{eq:mhnbs-1}
  \end{align}
  We bound the second term as
  \begin{align}
        \norm{\tilde{\Vb}_k(\Sigmab^{-2} -
          \tilde{\Sigmab}^{-2})\tilde{\Vb}_k^T} & =
        \max_{i \in [k]}|\sigma_i^{-2} - \tilde{\sigma}_i^{-2}| = \max_{i \in [k]}\frac{|\sigma_i^2 -
          \tilde{\sigma}_i^2|}{\sigma_i^2\tilde{\sigma}_i^2} \leq
        \frac{\mu\sigma_1^2}{\sigma_k^2\tilde{\sigma}_k^2}
  \end{align}
  where we use Weyl's inequality to bound $\sigma_i^2 -  \tilde{\sigma}_i^2$.
  Using Weyl's inequality and the fact that $\mu \geq 1/(20k\kappa_k)$,
  \begin{align}
    \tilde{\sigma}_k^2 & \geq \sigma_k^2 - \mu\sigma_1^2 \geq
  \sigma_k^2 - \sigma_k^2/10 k \geq \sigma_k^2/2,\notag\\
    \frac{\mu\sigma_1^2}{\sigma_k^2\tilde{\sigma}_k^2} & \leq \frac{2\mu\sigma_1^2}{\sigma_k^4}  =  \frac{2\mu \kappa_k}{\sigma_k^2}\label{eq:mhnbs-2}
  \end{align}

  Plugging Equations \eqref{eq:mhnbs-1} and \eqref{eq:mhnbs-2} into
  Equation \eqref{eq:mhnbs-0} gives
    \begin{align}
    \norm{\Vb_k\Sigmab^{-2}\Vb_k^T -
      \tilde{\Vb}_k\tilde{\Sigmab}^{-2}\tilde{\Vb}_k^T} & \leq
    \frac{1}{\sigma_k^2}(8(\mu k \kappa_k)^{1/3} + 2\mu \kappa_k) \leq
    \frac{10}{\sigma_k^2}(\mu k \kappa_k)^{1/3}\label{eq:use-lemma}\\
    & \leq \frac{10}{\sigma_k^2}
    \Big(\frac{\eps^3k^3}{10^3\sr(\Ab)^3\kappa_k^3}\Big)^{1/3} \leq \frac{k\eps}{\sigma_k^2\kappa_k \sr(\Ab)} = \frac{k \eps}{\|\Ab\|_F^2}.\nonumber
    \end{align}
    Equation \eqref{eq:use-lemma} follows by Theorem  \ref{thm:sigma2}.
    The conditions on $\mu$ needed to apply it are
    guaranteed by our choice of $\mu$ and $\eps$.
\end{proof}

To interpret this guarantee, consider the setting when all the points
have roughly the same $2$-norm. More precisely, if for some constant $C$
\[ \frac{\sqnorm{\rowa{i}}}{\|\Ab\|_F^2} \leq \frac{C}{i} \]
then Equation \eqref{eq:l-guarantee} gives
  \begin{align*}
    |L^k(i) - \td{L}^k(i)| \leq C\eps {k}/{n}.
  \end{align*}
  Note that $k$ is a constant whereas $n$ grows as more points come
  in. As mentioned in the discussion following Theorem~\ref{thm:FD}, the points which are considered
  outliers are those where $L^k(i) \gg \frac{k}{n}$. 
  For the parameters setting, if we let $\kappa_k =O(1)$ and $\sr(A) =
\Theta(k)$, then our bound on $\epsilon$ reduces to $\eps \leq
\min(\Delta, 1/k)$, and our choice of $\mu$ reduces to $\mu \approx \eps^3/k$.

To efficiently compute a sketch that satisfies \eqref{eq:cov}, we can use Frequent
Directions \citep{liberty2013simple}. We use the improved analysis of Frequent Directions in \citet{Liberty16}:
\begin{theorem}\label{thm:fd}
	\citep{Liberty16}
	There is an algorithm that takes the rows of $\Ab$ in
	streaming fashion and computes a sketch $\td{\Ab} \in \R^{\ell
		\times d}$ satisfying Equation \eqref{eq:cov} where $\ell =
	\sum_{i = k+1}^n\sigma_i^2/(\sigma_1^2 \mu)$.
\end{theorem}
Let $\td{\Ab} = \td{\Ub}\td{\Sigmab}\td{\Vb}^T$.
The algorithm  maintains $\td{\Sigmab}, \td{\Vb}$. It uses
$O(d\ell)$ memory and requires time at most $O(d\ell^2)$ per row. The
total time for processing $n$ rows is $O(nd\ell)$. If  $\ell \ll d$,
this is a significant improvement over the naive algorithm in
both memory and time. If we use Frequent directions, we set
\begin{align}
  \label{eq:fd}
  \ell = \frac{\sum_{i= k+1}^d\sigma_i^2}{\sigma_1^2\mu}
  \end{align}
where $\mu$ is set according to Theorem \ref{thm:t-online} and \ref{thm:l-online}. This leads to $\ell=\poly(k,\sr(\Ab),\kappa_k,\Delta,\epsilon^{-1})$. Note that this does not depend on $d$, and hence is considerably smaller for our parameter settings of interest.

\subsection{The Online Setting}\label{sec:online}

We now consider the online scenario where the leverage scores and projection distances are defined only  with respect to the input seen so far. Consider again the motivating example where each machine in a data center produces streams of measurements. Here, it is desirable to determine the anomaly score of a data point online as it arrives, with respect to the data produced so far, and by a streaming algorithm. We first define the online anomaly measures. Let $\Ab \in \R^{(i-1) \times d}$ denote the matrix of
points that have arrived so far (excluding $\rowa{i}$) and let $\Ub
\Sigmab \Vb^T$ be its SVD. Write $\rowa{i}$ in the basis of right
singular vectors as $\rowa{i} = \sum_{j = 1}^d\alpha_j\colof{v}{j}$. We
define its rank-$k$ leverage score and projection distance
respectively as
\begin{align}
l^k(i) & = \sum_{j = 1}^k \frac{\alpha_j^2}{\sigma_j^2} = \sqnorm{\Sigmab_k^{-1}\Vb^T_k\rowa{i}},\label{eq:l-online}\\
t^k(i) & = \sum_{j = k+1}^d \alpha_j^2 = \rowa{i}^T(\Ib -\Vb_k\Vb^T_k)\rowa{i}.\label{eq:t-online}
\end{align}

Note that in the online setting there is a one-pass streaming algorithm that can
compute both these scores, using time $O(d^3)$ per
row and $O(d^2)$ memory. This algorithm maintains the $d \times d$ covariance matrix $\Ab^T\Ab$ and
computes its SVD to get $\Vb_k, \Sigmab_k$. From these, it is easy to
compute both $l^k$ and $t^k$.

All our guarantees from the previous subsection directly carry over to this online scenario, allowing us to significantly improve over this baseline. This is because the guarantees are pointwise, hence they also hold for every data point if the scores are only defined for the input seen so far. This implies a one-pass algorithm which can approximately compute the anomaly scores (i.e., satisfies the guarantees in Theorem \ref{thm:t-online} and \ref{thm:l-online}) and uses space $O(d\ell)$ and requires time $O(nd\ell)$ for $\ell=\poly(k,\sr(\Ab),\kappa_k,\Delta,\epsilon^{-1})$ (independent of $d$).

The $\Omega(d)$ lower bounds in Section \ref{sec:lower_bound} show that
one cannot hope for sublinear dependence on $d$ for pointwise
estimates. In the next section, we show how to eliminate the
dependence on $d$ in the space requirement of the algorithm in exchange
for weaker guarantees.

\section{Average-case Guarantees for Anomaly Detection}
\label{sec:avg}

In this section, we present an approach which
circumvents the $\Omega(d)$ lower bounds by relaxing the pointwise
approximation guarantee. 

Let $\Ab = \Ub\Sigmab\Vb^T$ be the SVD of $\Ab$. The outlier scores we wish to compute are
\begin{align}
L^k(i) & = \sqnorm{e_i^T\Ub_k} = \sqnorm{\Sigmab_k^{-1}\Vb_k^T\rowa{i}}\label{eq:L-batch0},\\
T^k(i) & = \sqnorm{\rowa{i}} - \sqnorm{e_i^T\Ub_k\Sigmab_k} = \sqnorm{\rowa{i}} -  \sqnorm{\Vb_k^T\rowa{i}}\label{eq:T-batch0}.
\end{align}
Note that these scores are defined with respect to the principal space of the entire matrix.
We present a guarantee for any sketch
$\td{\Ab} \in R^{n \times \ell}$ that approximates the
column space of $\Ab$, or equivalently the covariance matrix
$\Ab\Ab^T$ of the row vectors. We can work with any sketch $\Ab$ where
\begin{align}
  \label{eq:cov2}
  \norm{\Ab\Ab^T - \td{\Ab}\td{\Ab}^T} \leq \mu \sigma_1^2.
\end{align}
Theorem \ref{thm:rp} stated in Section \ref{sec:rp} shows that such a sketch can be obtained for
instance by a random projection $\Rb$ onto $\R^\ell$ for $\ell =
\sr(\Ab)/\mu^2$: let $\td{\Ab} = \Ab\Rb$ for $\Rb \in \R^{d \times
  \ell}$ chosen from an appropriately family of random matrices. However, we need to be careful in our choice of the family of random matrices, as naively storing a $(d \times \ell)$ matrix requires $O(d\ell)$ space, which would increase the space requirement of our streaming algorithm. For example, if we were to choose each entry of $\Rb$ to be i.i.d. be $\pm \frac{1}{\sqrt{\ell}}$ , then we would need to store $O(d\ell)$ random bits corresponding to each entry of $\Rb$.

But Theorem \ref{thm:rp} also shows that this is unnecessary and we do not  need to explicitly store  a $(d \times \ell)$ random matrix. The guarantees of Theorem \ref{thm:rp} also hold when $\Rb$ is a pseudorandom matrix with the entries being $\pm \frac{1}{\sqrt{\ell}}$ with $\log(\sr(\Ab)/\delta)$-wise independence instead of full independence . Therefore, by using a simple polynomial based hashing scheme \citep{vadhan2012pseudorandomness} we can get the same sketching guarantees using only $O(\log(d)\log(\sr(\Ab)/\delta))$ random bits and hence only $O(\log(d)\log(\sr(\Ab)/\delta))$ space. Note that each entry of $\Rb$ can be computed from this random seed in time $O(\log(d)\log(\sr(\Ab)/\delta))$.

Theorem \ref{thm:sample} stated in Section \ref{sec:subsample} shows that such a sketch can also be obtained for a length-squared sub-sampling of the columns of the matrix, for $\ell=\tilde{O}(\sr(\Ab)/\mu^2)$ (where the $\tilde{O}$ hides logarithmic factors).

Given such a sketch, we expect $\td{\Ub}_k$ to be a good approximation to
$\Ub_k$. So we define our approximations in the natural way:
\begin{align}
  \td{L}^k(i) &  = \sqnorm{e_i^T\td{\Ub}_k^i}\label{eq:L-batch1},\\
  \td{T}^k(i) &  = \sqnorm{\rowa{i}} - \sqnorm{e_i^T\td{\Ub}_k\td{\Sigmab}_k}.\label{eq:T-batch1}
\end{align}

The analysis then relies on the machinery from Section \ref{sec:prelims}.
However, $\td{\Ub}_k$ lies in $\R^{n \times k}$ which is too costly to compute
and store, whereas $\td{\Vb}_k$ in contrast lies in $\R^{\ell \times
  k}$, for $\ell = 1/\mu^2\max(\sr(\Ab), \log(1/\delta))$. In
particular, both $\ell$ and $k$ are independent of $n, d$ and could be
significantly smaller. In many settings of practical interest, we have
$\sr(\Ab) \approx k$ and both are constants independent of $n, d$.
So in our algorithm, we use the following equivalent definition in
terms of $\td{\Vb}_k, \td{\Sigmab}_k$.
\begin{align}
  \td{L}^k(i) & = \sqnorm{\td{\Sigmab}_k^{-1}\td{\Vb}^T_k(\Rb^T\rowa{i})},\label{eq:L-batch2}\\
  \td{T}^k(i) &  = \sqnorm{\rowa{i}} - \sqnorm{\td{\Vb}^T_k(\Rb^T\rowa{i})}.\label{eq:T-batch2}
\end{align}

For the random projection algorithm, we compute $\td{\Ab}^T\td{\Ab}$ in $\R^{\ell \times \ell}$ in the
first pass, and then run SVD on it to compute $\td{\Sigmab}_k \in
\R^{k \times k}$ and $\td{\Vb}_k \in \R^{\ell \times k}$. Then in the
second pass, we use these to we compute $\td{L}^k$ and $\td{T}^k$. The total memory needed in the first pass is $O(\ell^2)$ for the
covariance matrix. In the second pass, we
need $O(k\ell)$ memory for storing $\Vb_k$. We also need $O(\log(d)\log(\sr(\Ab)/\delta))$ additional memory for storing the random seed from which the entries of $\Rb$ can be computed efficiently.

\subsection{Our Approximation Guarantees}

We now turn to the guarantees. Given the $\Omega(d)$ lower bound from Section
\ref{sec:lower_bound}, we cannot hope for a strong pointwise
guarantee, rather we will show a guarantee that hold on average, or
for most points.

The following simple Lemma bounds the sum of absolute
values of diagonal entries in symmetric matrices.
\begin{lemma}\label{lem:tr_abs}
	Let $\Ab \in \R^{n \times n}$ be symmetric. Then
	\[ \sum_{i=1}^{n}\Big| {\eb_i^T\Ab \eb_i}\Big| \leq \rk(\Ab)\|\Ab\|. \]
\end{lemma}
\begin{proof}
	Consider the eigenvalue decomposition of $\Ab=\Qb \Lambdab
	\Qb^T$ where $\Lambdab$ is the diagonal matrix of eigenvalues
	of $\Ab$, and $\Qb$ has orthonormal columns, so $\normfsq{\Qb} =
	\rk(\Ab)$. We can write,
	\begin{align*}
	\sum_{i=1}^{n}\Big| {\eb_i^T\Ab \eb_i}\Big|
	& = \sum_{i=1}^{n}\Big| {\eb_i^T\Qb \Lambdab \Qb^T
		\eb_i}\Big| = \sum_{i=1}^{n}\sum_{j=1}^{n}\Big| \Lambdab_{i,i}\Qb_{i,j}^2\Big| \\
	& \le \norm{\Ab} \sum_{i,j=1}^{n}\Big| \Qb_{i,j}^2\Big|  =
	\norm{\Ab} \normfsq{\Qb} = \rk(\Ab)\norm{\Ab}.
	\end{align*}
\end{proof}

We first state and prove Lemma \ref{lem:batch-mhnbs}
which bounds the average error in estimating $L^k$.

\begin{lemma}
	\label{lem:batch-mhnbs}
	Assume that $\Ab$ is $(k, \Delta)$-separated. Let $\td{\Ab}$ satisfy
	Equation \eqref{eq:cov2} for $\mu = \eps^2\Delta/16$ where $\eps <1$. Then
	\begin{align*}
	\sum_{i=1}^n |L^k(i) - \td{L}^k(i)| \leq  \eps \sum_{i=1}^n L^k(i).
	\end{align*}
\end{lemma}
\begin{proof}
	By Equations \eqref{eq:L-batch0} and \eqref{eq:L-batch1}
	\begin{align*}
	\sum_{i=1}^n |L^k(i) - \td{L}^k(i)| =
	\sum_{i=1}^n|e_i^t\Ub_k\Ub_k^Te_i - e_i^t\tilde{\Ub}_k\tilde{\Ub}_k^Te_i|.
	\end{align*}
	Let $\Cb = \Ub_k\Ub^T_k -  \tilde{\Ub}_k\tilde{\Ub}_k^T$, so that $\rank(\Cb) \leq 2k$. By
	Lemma \ref{lem:wed} (which applies since $\mu \leq \Delta/16$), we have
	\[ \norm{\Cb} \leq 2\sqrt{\frac{\mu}{\Delta}} \leq \frac{\eps}{2}.\]
	So applying Lemma \ref{lem:tr_abs}, we get
	\[
	\sum_{i=1}^n |e_i^t\Ub_k\Ub_k^Te_i -
	e_i^t\tilde{\Ub}_k\tilde{\Ub}_k^Te_i| \leq \frac{\eps}{2} 2k = \eps
	k.
	\]
	The claim follows by noting that the columns of $\Ub_k$ are
        orthonormal, so $\sum_{i=1}^n L^k(i) = k$.
\end{proof}

The guarantee above shows that the average additive error in
estimating $L^k(i)$ is $\frac{\eps}{n} \sum_{i=1}^n L^k(i)$ for a
suitable $\eps$. Note that the average value of $L^k(i)$ is
$\frac{1}{n} \sum_{i=1}^n L^k(i)$, hence we obtain small additive
errors on average. Additive error guarantees for $L^k(i)$ translate to
multiplicative guarantees as long as $L^k(i)$ is not too small, but
for outlier detection the candidate outliers are those points for
which $L^k(i)$ is large, hence additive error guarantees are
meaningful for preserving outlier scores for points which could be
outliers.

Similarly, Lemma \ref{lem:batch-proj} bounds the average error in estimating $T^k$.

\begin{lemma}
	\label{lem:batch-proj}
	Assume that $\Ab$ is $(k, \Delta)$-separated. Let
	\begin{align}
	\label{eq:ugly-4}
	\eps \leq \frac{\min(\Delta k^2, k)}{\sr(A)}.
	\end{align}
	Let $\td{\Ab}$ satisfy
	Equation \eqref{eq:cov2} for
	\begin{align}
	\label{eq:ugly-3}
	\mu = \frac{\eps^3\sr(\Ab)^3}{125 k^4}.
	\end{align}
	Then
	\begin{align*}
	\sum_{i=1}^n |T^k(i) - \td{T}^k(i)| \leq \frac{\eps \normfsq{\Ab}}{\normfsq{\Ab-\Ab_k}} \sum_{i=1}^nT^k(i).
	\end{align*}
\end{lemma}
\begin{proof}
	By Equations \eqref{eq:T-batch0} and \eqref{eq:T-batch1}, we have
	\begin{align*}
	\sum_{i=1}^n |T^k(i) - \td{T}^k(i)| = \sum_{i=1}^n|\sqnorm{e_i^T\Ub_k\Sigmab_k} -
	\sqnorm{e_i^T\td{\Ub}_k\td{\Sigmab}_k}| = \sum_{i=1}^n|e_i^t\Ub_k\Sigmab_k^2\Ub_k^Te_i -
	e_i^t\tilde{\Ub}_k\tilde{\Sigmab}_k^2\tilde{\Ub}_k^Te_i|
	\end{align*}
	Let $\Cb = \Ub_k\Sigmab_k^2\Ub_k^T -
	\tilde{\Ub}_k\tilde{\Sigmab}_k^2\tilde{\Ub}_k^T$. Then $\rank(\Cb)
	\leq 2k$. We now bound its operator norm as follows
	\begin{align*}
	\norm{\Cb} \leq \norm{\Ub_k\Sigmab_k^2\Ub_k^T -
		\tilde{\Ub}_k\Sigmab_k^2\tilde{\Ub}_k^T} + \norm{
		\tilde{\Ub}_k(\Sigmab_k^2 - \tilde{\Sigmab}_k^2)\tilde{\Ub}_k^T}.
	\end{align*}
	To bound the first term, we use Theorem \ref{thm:sigma1} (the
	condition on $\mu$ holds by our choice of $\eps$ in Equation \eqref{eq:ugly-4} and $\mu$ in Equation \eqref{eq:ugly-3}) to get
	\[ \norm{\Ub_{\leq k}\Sigmab_k^2\Ub^t_{\leq k} - \tilde{\Ub}_{\leq
			k}\Sigmab_k^2\tilde{\Ub}^t_{\leq k}} \leq 4\sigma_1(\Ab)^2(\mu
	k)^{1/3}.\]
	For the second term, we use
	\[ \norm{\tilde{\Ub}_k(\Sigmab_k^2 -
		\tilde{\Sigmab}_k^2)\tilde{\Ub}_k^T} \leq \norm{\Sigmab_k^2 -
		\tilde{\Sigmab}_k^2} \leq \mu \sigma_1(\Ab)^2.
	\]
	Overall, we get $\norm{\Cb} \leq 5\sigma(\Ab)^2(\mu k)^{1/3}$. So
	applying Lemma \ref{lem:tr_abs}, we get
	\begin{align*}
	\sum_{i=1}^n |e_i^t\Ub_k\Sigmab_k^2\Ub_k^Te_i -
	e_i^t\tilde{\Ub}_k\tilde{\Sigmab}_k^2\tilde{\Ub}_k^Te_i| & \leq
	5\sigma(\Ab)^2(\mu k)^{1/3} \cdot 2k\\
	& \leq 5\sigma(\Ab)^2\mu^{1/3}k^{4/3}\\
	& \leq \eps \sr(\Ab)\sigma(\Ab)^2 = \eps\normfsq{\Ab}.
	\end{align*}
	In order to obtain the result in the form stated in the Lemma, note that $\sum_{i=1}^n T^k(i)=\normfsq{\Ab-\Ab_k}$.
\end{proof}

Typically, we expect $\normfsq{\Ab- \Ab_k}$ to be
a constant fraction of $\normfsq{\Ab}$. Hence the guarantee above
says that on average, we get good additive guarantees.

\subsection{Guarantees for Random Projections}\label{sec:rp}

\begin{theorem}\label{thm:rp}\citep{cohen2015optimal,cohen2015dimensionality}
	Consider any matrix $\Ab
	\in \mathbb{R}^{n \times d}$. Let $\Rb = (1/\sqrt{\ell})\Xb$ where $\Xb \in \R^{d \times
		\ell}$ is a random matrix drawn from any of the following distributions of matrices. Let $\tilde{\Ab} = \Ab\Rb$. Then
	with probability $1-\delta$, 	
	\begin{align*}
	\norm{\Ab\Ab^T-\tilde{\Ab}\tilde{\Ab}^T}\le \mu \norm{\Ab}^2.
	\end{align*}
\begin{enumerate}
	\item $\Rb$ is a dense random matrix with each entry being an
	i.i.d. sub-Gaussian random variable and $\ell = O(\frac{\sr(A) + \log(1/\delta)}{\epsilon^2})$.
	\item $\Rb$ is a fully sparse embedding matrix , where
	each column has a single $\pm 1$ in a random position (sign and position chosen uniformly and
	independently) and $\ell = O(\frac{\sr(A)^2}{\epsilon^2\delta})$. Additionally, the same matrix family except where the position and sign for
	each column are determined by a 4-independent hash function.
	\item $\Rb$ is a Subsampled Randomized Hadamard Transform (SRHT) \citep{ailon2009fast} with $\ell = O(\frac{\sr(A) + \log(1/(\epsilon\delta))}{\epsilon^2})\log(\sr(A)/\delta)$.
	\item $\Rb$ is a dense random matrix with each entry being $\pm \sqrt{\frac{1}{\ell}}$ for $\ell = O(\frac{\sr(A) + \log(1/\delta)}{\epsilon^2})$ and the entries are drawn from a $\log(\sr(A)/\delta)$-wise independent family of hash functions. Such a hash family can be constructed with $O(\log(d)\log(\sr(A)/\delta))$ random bits use standard techniques (see for e.g. \citet{vadhan2012pseudorandomness} Sec 3.5.5).
\end{enumerate}	
	
\end{theorem}

Using Theorem \ref{thm:rp} along with Lemma \ref{lem:batch-mhnbs} and Lemma \ref{lem:batch-proj} and with the condition that $\sr(A)=O(k)$ gives Theorem \ref{thm:random_proj} from the introduction---which shows that taking a random projection with $\ell = k^3 \cdot \poly(\Delta, \eps^{-1})$ ensures that the error guarantees in Lemma \ref{lem:batch-mhnbs} and Lemma \ref{lem:batch-proj} hold with high probability.

\eat{

The following theorem from \citet{koltchinskii2017}
implies that random projections give good sketches of a matrix.

\begin{theorem}\label{thm:cov}\citep{koltchinskii2017}
	Consider any matrix $\Ab
	\in \mathbb{R}^{n \times d}$. Let $\Rb = (1/\sqrt{\ell})\Xb$ where $\Xb \in \R^{d \times
		\ell}$ is a random matrix with each entry being an
	i.i.d. sub-Gaussian random variable. Let $\tilde{\Ab} = \Ab\Rb$. Then
	with probability $1-e^{-t}$, for some fixed constant $C$,
	\begin{align*}
	\norm{\Ab\Ab^T-\tilde{\Ab}\tilde{\Ab}^T}\le
	C\norm{\Ab\Ab^T}\Big(\sqrt{\frac{\sr(\Ab)}{\ell}} +
	{\frac{\sr(\Ab)}{\ell}}+ \sqrt{\frac{t}{\ell}}+ {\frac{t}{\ell}} \Big).
	\end{align*}
\end{theorem}

It has the following corollary.
\begin{corollary}\label{cor:covariance}
	Let $\ell \geq C'\mu^{-2}\max(\sr(\Ab), \log(1/\delta))$. Then with
	probability $1 - \delta$,
	\begin{align*}
	\norm{\Ab\Ab^T-\tilde{\Ab}\tilde{\Ab}^T}\le \mu \norm{\Ab}^2.
	\end{align*}
\end{corollary}

Using Corollary \ref{cor:covariance} along with Lemma \ref{lem:batch-mhnbs} and Lemma \ref{lem:batch-proj} gives Theorem \ref{thm:random_proj} from the introduction---which shows that taking a random projection with $\ell = \poly(k, \eps^{-1}, \sr(A))$ ensures that the error guarantees in Lemma \ref{lem:batch-mhnbs} and Lemma \ref{lem:batch-proj} hold with high probability.

}
\subsection{Results on Subsampling Based Sketches}\label{sec:subsample}

Subsampling based sketches can yield both row and column space
approximations. The algorithm for preserving the row space,
i.e. approximating $\Ab^T\Ab$,  using row subsampling is
straightforward. The sketch samples $\ell$ rows of $\Ab$ proportional
to their squared lengths to obtain a sketch $\td{\Ab}\in
\mathbb{R}^{\ell \times d}$. This can be done in a single pass in a
row streaming model using reservior sampling. Our streaming algorithm
for approximating anomaly scores using row subsampling follows the
same outline as Algorithm \ref{alg:FD} for Frequent Directions. We
also obtain the same guarantees as Theorem~\ref{thm:FD} for Frequent
Directions, using the guarantee for subsampling sketches stated at the
end of the section (in Theorem ~\ref{thm:sample}). The guarantees in
Theorem~\ref{thm:FD} are satisfied by subsampling $\ell=k^2\cdot
\poly(\kappa,\eps^{-1},\Delta)$ columns, and the algorithm needs
$O(d\ell)$ space and $O(nd)$ time.

In order to preserve the column space using subsampling,
i.e. approximate $\Ab\Ab^T$, we need to subsample the columns of the
matrix. Our streaming algorithm follows a similar outline as
Algorithm~\ref{alg:random} which does a random projections of the rows
and also approximates $\Ab\Ab^T$. However, there is a subtlety
involved. We need to subsample the columns, but the matrix is arrives
in row-streaming order. We show that using an additional pass, we can
subsample the columns of $\Ab$ based on their squared lengths. This
additional pass does reservoir sampling on the squared entries of the
matrix, and uses the column index of the sampled entry as the column
to be subsampled. The algorithm is stated in Algorithm
\ref{alg:subsample}. It requires space $O(\ell \log d)$ in order to
store the $\ell$ indices to subsample, and space $O(\ell^2)$ to store
the covariance matrix of the subsampled data. Using the guarantees for
subsampling in Theorem ~\ref{thm:sample}, we can get the same
guarantees for approximating anomaly scores as for a random projection
of the rows. The guarantees for random projection in Theorem~\ref{thm:rp} are satisfied by subsampling $\ell=k^3\cdot
\poly(\eps^{-1},\Delta)$ columns, and the algorithm needs
$O(d\ell+\log d)$ space and $O(nd)$ time.

\begin{algorithm}
	\DontPrintSemicolon
	\SetAlgoLined
	\SetAlgoNoEnd
	\SetKwFunction{FDot}{Dot}
	\SetKwFunction{FUpdate}{Update}

	\newcommand\mycommfont[1]{\rmfamily{#1}}
	\SetCommentSty{mycommfont}
	\SetKwComment{Comment}{$\triangleright$ }{}

	\Input{Choice of $k$ and $\ell$.}
	\Init{}{
		Set covariance $\tilde{\Ab}^T\tilde{\Ab}\leftarrow 0$\;
		For $1\le t\le \ell$, set $S_t=1$
		 { \Comment*[f]{$S_i$ stores the $\ell$ column indices we will subsample}}\;
		Set $s\rightarrow 0  $   \Comment*[f]{$s$ stores the sum of the squares of entries seen so far}\;
		}

	\ZeroPass{As each element $a_{ij}$ of $\Ab$ streams in,}{
		 Update $s\rightarrow s + a_{ij}^2$\;
		 \For{$1\le t\le \ell$}{Set $S_t\rightarrow j$ with probability $a_{ij}^2/s$ }

	}
	\FirstPass{As each row $a_{(i)}$ streams in,}{
		Project by $\Rb$ to get $\Rb^T a_{(i)}$  \;
		Update covariance $\tilde{\Ab}^T\tilde{\Ab}\leftarrow \tilde{\Ab}^T\tilde{\Ab} + (\Rb^T a_{(i)})(\Rb^T a_{(i)})^T $
	}
	\SVD{}{
		Compute the top $k$ right singular vectors of $\tilde{\Ab}^T\tilde{\Ab}$
	}
	\SecondPass{As each row $a_{(i)}$ streams in,}{
		Project by $\Rb$ to get $\Rb^T a_{(i)}$  \;
		For each projected row, use the estimated right singular vectors to compute the leverage scores and projection distances \;
	}

	\caption{Algorithm to approximate anomaly scores using column subsampling}
	\label{alg:subsample}
\end{algorithm}

\paragraph{Guarantees for subsampling based sketches.}

\citet{drineas2006fast} showed that sampling columns proportional to
their squared lengths approximates $\Ab\Ab^T$ with high
probability. They show a stronger Frobenius norm guarantee than the
operator norm guarantee that we need, but this worsens the dependence
on the stable rank. We will instead use the following guarantee due to
\citet{magen2011low}.

\begin{theorem}\citep{magen2011low}\label{thm:sample}
	Consider any matrix $\Ab
	\in \mathbb{R}^{n \times d}$. Let $\td{\Ab}\in \mathbb{R}^{n \times \ell}$ be a matrix obtained by subsampling the columns of $\Ab$ with probability proportional to their squared lengths. Then with probability $1-1/\poly{(\sr(\Ab))}$, for $\ell \ge \sr(\Ab)\log(\sr(\Ab)/\mu^2)/\mu^2$
	\begin{align*}
	\norm{\Ab\Ab^T-\tilde{\Ab}\tilde{\Ab}^T}\le \mu \norm{\Ab}^2.
	\end{align*}

\end{theorem}




\vspace{-8pt}
\section{Conclusion}
\vspace{-6pt}
We show that techniques from sketching can be used to derive
simple and practical algorithms for computing subspace-based anomaly scores which
provably approximate the true scores at a significantly lower cost in terms of
time and memory. A promising direction of future work is to use them in real-world
high-dimensional anomaly detection tasks.

\eat{

Another resource that this optimizes is communication,
which is important in the IoT regime where the rows
corresponds to data collected by different sensors, and the
communication cost of transmitting the data from the sensors to a
central server is a bottleneck, and where anomaly detection is an important problem.
\eat{In such a scenario, the sensors can
	compress the  data by subsampling or random projection before sending
	them to the server. The server would reply with (the SVD of) the
	aggregate covariance, which the sensors could then use to compute the
	anomaly scores locally.}
}

\section*{Acknowledgments}

The authors thank David Woodruff for suggestions on using communication complexity tools to show lower bounds on memory usage for approximating anomaly scores and Weihao Kong for several useful discussions on estimating singular values and vectors using random projections. We also thank Steve Mussmann, Neha Gupta, Yair Carmon and the anonymous reviewers for detailed feedback on initial versions of the paper. VS's contribution was partially supported by NSF award 1813049, and ONR award N00014-18-1-2295.

\bibliographystyle{unsrtnat}
\bibliography{references.bib}
\appendix



\section{Proofs for Section \ref{sec:prelims}}\label{sec:prelim_full}
We will prove bounds for items (2) and (3) listed in the beginning of Section \ref{sec:prelims}, which
are given by Theorems \ref{thm:sigma1} and \ref{thm:sigma2}
respectively. To prove these, the next two technical lemmas give
perturbation bounds on the operator norm of positive semi-definite
matrices of the from $\Ub_k\Lambdab\Ub^T_k$, where $\Lambdab$ is a
diagonal matrix with non-negative entries. In order to do this, we split the matrix $\Sigmab$ to a well-separated component and a small residual component. 

We now describe the decomposition of $\Sigmab$ .
Let $\delta$ be  a parameter so that
\begin{align}
\label{eq:choose-delta}
6\mu \leq \delta \leq \Delta
\end{align}
We partition the indices $[k]$ into a set of disjoint intervals $B(\Ab,
\delta) = \{B_1,\ldots,B_m\}$ based on the singular values of $\Ab$ so
that there is a separation of at least $\delta \sigma_1^2$ between intervals, and at most $\delta
\sigma_1^2$ within an interval. Formally, we start with $i =1$
assigned to $B_1$. For $i \geq 2$, assume that we have assigned $i-1$ to $B_j$. If
 \[ \sigma^2_{i}(\Ab) - \sigma^2_{i-1}(\Ab) \leq \delta
 \sigma_1^2(\Ab) \]
then $i$ is also assigned to $B_j$, whereas if
\[ \sigma^2_i(\Ab) - \sigma^2_{i-1}(\Ab) >  \delta \sigma_1^2(\Ab)\]
then it is assigned to a new bucket $B_{j+1}$. Let $b_j$ denote the
largest index in the interval $B_j$ for $j \in [m]$.

	\begin{figure}[hb]
		\centering
		\includegraphics[width=4 in]{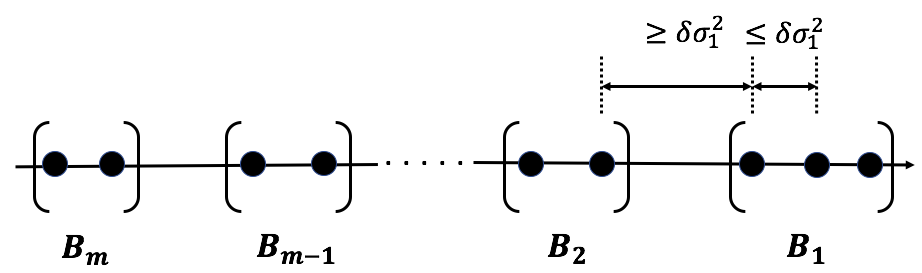}
		\caption{Illustration of the decomposition of the $k$ singular values into $m$ intervals such that there is a separation of at least $\delta \sigma_1^2$ between intervals, and at most $\delta
			\sigma_1^2$ within an interval.}
		\label{fig:interval}
	\end{figure}


Let $\Lambdab = \diag(\lambda_1,\ldots\lambda_k)$ be a diagonal matrix
with all non-negative entries which is constant on each interval $B_j$ and non-increasing across
intervals. In other words, if $i \geq j$, then  $\lambda_i \geq
\lambda_j$, with equality holding whenever $i, j$ belong to the same
interval $B_k$.  Call such a matrix a \emph{diagonal non-decreasing
matrix with respect to $B(\Ab, \delta)$}. Similarly, we define diagonal
non-increasing matrices with respect to $B(\Ab,\delta)$ to be
non-increasing but are constant on each interval $B_k$. The following
is the main technical lemma in this section, and  handles the case where the diagonal matrix is well-separated. It is a generalization of Lemma \ref{lem:diag2_simple}. In Lemma \ref{lem:diag2_simple} we assumed that the eigenvalues in each bucket are equal, here we generalize to the case where the eigenvalues in each bucket are separated by at most $\delta \sigma_1^2(\Ab)$.

\begin{lemma}
  \label{lem:diag2}
  Let $6\mu \leq \delta \leq \Delta$. Let $\Lambdab =
  \diag(\lambda_1,\ldots\lambda_k)$ be a diagonal non-increasing or a
  diagonal non-decreasing matrix with respect to $B(\Ab, \delta)$. Then
  \[ \norm{\Ub_k\Lambdab\Ub^T_{k} - \td{\Ub}_k\Lambdab\td{\Ub}^T_{k}}
  \leq 4\norm{\Lambdab}\sqrt{\frac{\mu}{\delta}}.\]
\end{lemma}
\begin{proof}
  Let us set $\lambda_{b_{m+1}} =  0$ for convenience. Since
  \begin{align*}
    \Ub_{b_j}\Ub_{b_j}^T  = \sum_{i =
      1}^{b_j}\colof{u}{i}{\colof{u}{i}}^T
  \end{align*}
  we can write
  \begin{align*}
    \Ub_{k} \Lambdab \Ub_{k}^T  = \sum_{j=1}^m\lambda_{b_j}\sum_{i
      \in B_j}\colof{u}{i}{\colof{u}{i}}^T = \sum_{j=1}^m(\lambda_{b_j}
    - \lambda_{b_{j+1}})\Ub_{b_j}\Ub_{b_j}^T
  \end{align*}
  and similarly for $\td{\Ub}_{k} \Lambdab\td{\Ub}_{k}^T$. So we can
  write
  \begin{align*}
    \Ub_{k} \Lambdab \Ub_{k}^T -  \td{\Ub}_{k} \Lambdab\td{\Ub}_{k}^T
    = \sum_{j=1}^{m}(\lambda_{b_j}-\lambda_{b_{j+1}})(\Ub_{b_j}\Ub_{b_j}^T-\td{\Ub}_{b_j}\td{\Ub}_{b_j}^T).
  \end{align*}
  Therefore, by the triangle inequality and Lemma \ref{lem:non_dec1},
  \begin{align*}
    \norm{ \Ub_{k} \Lambdab \Ub_{k}^T -  \td{\Ub}_{k} \Lambdab\td{\Ub}_{k}^T}
    =
    \sum_{j=1}^{m}|\lambda_{b_j}-\lambda_{b_{j+1}}|\norm{(\Ub_{b_j}\Ub_{b_j}^T-\td{\Ub}_{b_j}\td{\Ub}_{b_j}^T)}
      \leq 2\sqrt\frac{\mu}{\delta} \sum_{j=1}^m|\lambda_{b_j}-\lambda_{b_{j+1}}|
  \end{align*}
  thus proving that $\sum_{j=1}^m|\lambda_{b_j}-\lambda_{b_{j+1}}|
  \leq 2 \norm{\Lambdab}$ would imply the claim.

  When $\Lambdab$ is diagonal non-increasing with respect to
  $B(\Ab, \delta)$, then $\lambda_{b_j} -
  \lambda_{b_{j+1}} \geq 0$ for all $j \in [m]$, and $\norm{\Lambdab} =
    \lambda_{b_1}$. Hence
    \[ \sum_{j=1}^m|\lambda_{b_j}-\lambda_{b_{j+1}}| =
    \sum_{j=1}^m(\lambda_{b_j}-\lambda_{b_{j+1}}) =  \lambda_{b_1} -
    \lambda_{b_{m+1}} = \norm{\Lambdab}.\]

  When $\Lambdab$ is diagonal non-decreasing with respect to
  $B(\Ab, \delta)$, then for $j \leq m -1$, $\lambda_{b_j} \leq \lambda_{b_{j+1}}$, and $\norm{\Lambdab}
  = \lambda_{b_m}$. Hence
    \[ \sum_{j=1}^{m-1}|\lambda_{b_j}-\lambda_{b_{j+1}}| = \sum_{j=1}^{m-1} \lambda_{b_{j+1}}
      - \lambda_{b_j} = \lambda_{b_m} - \lambda_{b_1} \leq \lambda_{b_m} \]
      whereas $|\lambda_{b_m} - \lambda_{b_{m+1}}| = \lambda_{b_m}$. Thus overall,
        \[ \sum_{j=1}^m|\lambda_{b_j}-\lambda_{b_{j+1}}| \leq 2\lambda_{b_m}
      = 2\norm{\Lambdab}.\]
\end{proof}



We use this to prove our perturbation bound for $\Ub_k\Sigmab_k^2\Ub^T_k$.
\sigmaone*
\begin{proof}
Define $\Lambdab$ to be the $k\times k$ diagonal non-increasing matrix such that
all the entries in the interval $B_j$ are $\sigma_{b_j}^2$.
Define $\Omegab$ to be the $k\times k$
diagonal matrix such that $\Lambdab + \Omegab = \Sigmab_k^2$.
With this notation,
\begin{align}
  \label{eq:l+o}
  \norm{\Ub_{k} \Sigmab_k^2 \Ub_{k}^T -  \td{\Ub}_{k}
    {\Sigmab}_k^2\td{\Ub}_{k}^T}
  & = \norm{(\Ub_{k} \Lambdab \Ub_{k}^T -  \td{\Ub}_{k}
  \Lambdab\td{\Ub}_{k}^T) + (\Ub_{k} \Omegab \Ub_{k}^T -  \td{\Ub}_{k}
  \Omegab \td{\Ub}_{k}^T)}\notag\\
  &\leq \norm{\Ub_{k} \Lambdab \Ub_{k}^T -  \td{\Ub}_{k}
  \Lambdab\td{\Ub}_{k}^T} + \norm{\Ub_{k} \Omegab \Ub_{k}^T -  \td{\Ub}_{k}
  \Omegab \td{\Ub}_{k}^T)}
\end{align}

By definition, $\Lambdab$ is diagonal non-increasing,
$\norm{\Lambdab}=\sigma_{b_1}^2 \leq \sigma_1^2 $. Hence
by part (1) of Lemma \ref{lem:diag2},
\[   \norm{\Ub_{k} \Lambdab \Ub_{k}^T -  \td{\Ub}_{k}
  \Lambdab\td{\Ub}_{k}^T} \leq
4\sigma_1^2\sqrt{\frac{\mu}{\delta}}. \]

By our definition of the $B_j$s, if $i, i+1 \in B_j$
then $\sigma_i^2 - \sigma_{i+1}^2 \leq \delta \sigma_1^2$, hence for
any pair $i ,i' \ in B_j$, $\sigma_i^2 - \sigma_{i'}^2 \leq
k \delta \sigma_1^2$. Hence
\[ \norm{\Omegab} = \max_{i \in B_j}(\sigma_i^2 - \sigma_{b_j}^2)
\leq k\delta \sigma_1^2\]
We use Lemma \ref{lem:diag1} to get
\[ \norm{\Ub_{k} \Omegab \Ub_{k}^T -  \td{\Ub}_{k} \Omegab
  \td{\Ub}_{k}^T} \leq 2k\delta \sigma_1^2. \]
Plugging these bounds into Equation \eqref{eq:l+o}, we get
\begin{align*}
    \norm{\Ub_{k} \Sigmab_k^2 \Ub_{k}^T -
      \td{\Ub}_k\Sigmab_k^2\td{\Ub}_{k}^T} \leq
    4\sigma_1^2(\sqrt{\frac{\mu}{\delta}} + k \delta).
\end{align*}
We choose $\delta = \mu^{1/3}/k^{2/3}$ that minimizes the RHS, to get
\begin{align*}
\norm{	{\Ub_{k} \Sigmab_k^2 \Ub_{k}^T -  \td{\Ub}_{k}
    {\Sigmab}_k^2\td{\Ub}_{k}^T}} &\le  8\sigma_1^2(\mu k)^{1/3}.
\end{align*}
We need to ensure that this choice satisfies $6\mu \leq \delta$. This
holds since it is equivalent to $6(\mu k)^{2/3} \leq 1$, which is
implied by $\mu k \leq 1/20$. We need $\delta\leq \Delta$ which holds
since $\mu \leq \Delta^3k^2$.
\end{proof}

Next we derive our perturbation bound for
$\Ub_{k}\Sigmab_k^{-2}\Ub^T_{k}$, which will depend on the condition
number $\kappa =\sigma_1^2/\sigma_k^2$. The proof is similar to the
proof of Theorem \ref{thm:sigma1}.
\sigmatwo*
\begin{proof}
	We use a similar decomposition as in Theorem
	\ref{thm:sigma1}. Define $\Lambdab$ to be diagonal
        non-decreasing such that all the entries in the interval $B_j$
        are $1/\sigma_{b_j}^2$. Note that $\norm{\Lambdab} \leq
        \sigma_k^{-2}$.
        Using Lemma \ref{lem:diag2}, we get
	\begin{align}
          \label{eq:diag}
	\norm{\Ub_{k} \Lambdab_k^{-2} \Ub_{k}^T -  \td{\Ub}_k \Lambdab
          \td{\Ub}_{k}^T} \leq \frac{4}{\sigma_k^2}\sqrt{\frac{\mu}{\delta}}.
	\end{align}

        Define $\Omegab = \Sigmab_k^{-2}-\Lambdab$. Note that
        \begin{align*}
        \norm{\Omegab} = \max_{i \in B_j} \frac{1}{\sigma^2_{b_j}} - \frac{1}{\sigma_i^2}
        = \max_{i \in B_j} \frac{\sigma_i^2 -
          \sigma_{b_j}^2}{\sigma_i^2\sigma^2_{b_j}}
        \leq \frac{k\delta\sigma_1^2}{\sigma_k^4} = \frac{k \kappa
          \delta}{\sigma_k^2}.
        \end{align*}
        By using this in Lemma \ref{lem:diag1},
	\begin{align}
          \label{eq:noise}
          \norm{\Ub_{k} \Omegab \Ub_{k}^T -  \td{\Ub}_{k} \Omegab
            \td{\Ub}_{k}^T} \leq \frac{2k\kappa\delta}{\sigma_k^2}.
        \end{align}
        Putting Equation \eqref{eq:diag} and \eqref{eq:noise} together, we get
        \begin{align*}
        \norm{\Ub_{k}\Sigmab_k^{-2}\Ub^T_{k} -
          \td{\Ub}_k\Sigmab_k^{-2}\td{\Ub}^T_{k}} \leq
        \frac{4}{\sigma_k^2}\left(\sqrt{\frac{\mu}{\delta}} + k\kappa\delta\right).
        \end{align*}

	The optimum value of $\delta$ is $\mu^{1/3}/(k
	\kappa)^{2/3}$ which gives the claimed bound. A routine
        calculation shows that the condition $6 \mu \leq \delta \leq
	\Delta$ holds because  $\mu \leq \min(\Delta^3(k\kappa)^2, 1/(20 k \kappa))$.
\end{proof}

\section{Ridge Leverage Scores}\label{sec:ridge}

Regularizing the spectrum (or alternately, assuming that the data
itself has some ambient Gaussian noise) is closely tied to the notion
of ridge leverage scores \citep{Alaoui15}. Various versions of ridge
leverages had been shown to be good estimators for the Mahalanobis
distance in the high dimensional case and were demonstrated to be an
effective tool for anomaly detection \citep{Holgersson12}. There are efficient
sketches that approximate the ridge leverage score for specific values
of the parameter $\lambda$~\citep{CohenMM17}.

Recall that we measured deviation in the tail by the distance from the
principal $k$-dimensional subspace, given by
\[ T^k(i)= \sum_{j = k+1}^d\alpha_j^2.\]
We prefer this to using
\[ L^{> k}_i \defeq \sum_{j= k + 1}^d \frac{\alpha_j^2}{\sigma_j}^2\]
since it is more robust to the small $\sigma_j$,
and is easier to compute in the streaming
setting.\footnote{Although the latter measure is also studied in the
	literature and may be preferred in settings where there is structure in the tail.}

An alternative approach is to consider the ridge leverage
scores, which effectively replaces the covariance matrix $\Ab^T\Ab$
with $\Ab^T\Ab + \lambda \Ib$, which increases all the singular values by
$\lambda$, with the effect of damping the effect of small singular
values. We have
\[ L_\lambda(i) = \sum_{j=1}^d \frac{\alpha_j^2}{\sigma_j^2 + \lambda}.\]

Consider the case when the data is generated from a true
$k$-dimensional distribution, and then corrupted with a small amount of
white noise. It is easy to see that the data points
will satisfy both concentration and separation
assumptions. In this case, all the notions suggested above will
essentially converge. In this case, we expect $\sigma_{k+1}^2 \approx \cdots
\approx \sigma_n^2$. So
\[ L^{> k}_i = \sum_{j= k + 1}^d \frac{\alpha_j^2}{\sigma_j^2}
\approx \frac{T^k(i)}{\sigma_{k+1}^2}. \]
If $\lambda$ is chosen so that $\sigma^2_{k} \gg \lambda \gg \sigma_{k+1}^2$, it follows that
\[ L_\lambda(i) \approx L^k(i) + \frac{T^k(i)}{\lambda}. \]

\section{Streaming Lower Bounds}\label{sec:lower_bound}

In this section we prove lower bounds on streaming algorithms for computing  leverage scores, rank $k$ leverage scores and ridge leverage scores for small values of $\lambda$. Our lower bounds are based on reductions from the multi party set disjointness problem denoted as $\textsc{DISJ}_{t,d}$. In this problem, each of $t$ parties is given a set from the universe $[d]=\{1,2,\dots,d\}$, together with the promise that either the sets are \emph{uniquely intersecting}, i.e. all sets have exactly one element in common, or the sets are pairwise disjoint. The parties also have access to a common source of random bits.  \citet{chakrabarti2003near} showed a $\Omega(d/(t\log t))$ lower bound on the communication complexity of this problem. As usual, the lower bound on the communication in the set-disjointness problem translates to a lower bound on the space complexity of the streaming algorithm. 

\begin{theorem}\label{thm:lower_bound}
	For sufficiently large $d$ and $n\ge O(d)$, let the input matrix be $\Ab \in \mathbb{R}^{n \times d}$. Consider a row-wise streaming model the algorithm may make a constant number passes over the data.
	\begin{enumerate}
		\item Any randomized algorithm which computes a $\sqrt{t}$-approximation to all the leverage scores for every matrix $\Ab$ with probability at least $2/3$ and with $p$ passes over the data uses space $\Omega(d/(t^2p\log t))$.
		\item For $\lambda\le \frac{\sr(\Ab)}{2d} \sigma_1(\Ab)^2$, any randomized streaming algorithm which computes a $\sqrt{t/2}$-approximation to all the $\lambda$-ridge leverage scores for every matrix $\Ab$ with $p$ passes over the data with probability at least $2/3$ uses space $\Omega(d/(pt^2\log t))$.
		\item For $2\le k\le d/2$, any randomized streaming algorithm which computes any multiplicative approximation to all the rank $k$ leverage scores for every matrix $\Ab$ using $p$ passes and with probability at least $2/3$ uses space $\Omega(d/p)$.
		\item For $2\le k\le d/2$, any randomized algorithm which computes any multiplicative approximation to the distances from the principal $k$-dimensional subspace of every row for every matrix $\Ab$ with $p$ passes and with probability at least $2/3$ uses space $\Omega(d/p)$.
	\end{enumerate}
\end{theorem}
We make a few remarks:
\begin{itemize}
\item The lower bounds are independent of the stable rank of the matrix. Indeed they hold both when $\sr(\Ab)=o(d)$ and  when $\sr(\Ab)=\Theta(d)$. 
\item The Theorem is concerned only with the  working space; the algorithms are permitted to have separate access to a random string. 
\item In the first two cases an additional $\log t$ factor in the space requirement can be obtained if we limit the streaming  algorithm to one pass over the data.
\end{itemize}

Note that Theorem \ref{thm:lower_bound} shows that the Frequent Directions sketch for computing outlier scores is close to optimal as it uses $O(d\ell)$ space, where the projection dimension $\ell$ is a constant for many relevant parameter regimes. The lower bound also shows that the average case guarantees for the random projection based sketch which uses working space $O(\ell^2)$ cannot be improved to a point-wise approximation. We now prove Theorem \ref{thm:lower_bound}.

\begin{proof}
	We describe the reduction from $\textsc{DISJ}_{t,d}$ to computing each of the four quantities.
	
	\paragraph{$(1)$ Leverage scores:} Say for contradiction we have an algorithm which computes a $\sqrt{t}$-approximation to all the leverage scores for every matrix $\Ab \in \R^{n \times d}$ using space $O(d/(kt^2\log t))$ and $k=O(1)$ passes. We will use this algorithm to design a protocol for $\textsc{DISJ}_{t,d}$ using a communication complexity of $O(d/(t\log t))$.  In other words we need the following lemma.
	
	\begin{lemma}
		A streaming algorithm which approximates all leverage scores within $\sqrt t$ with $p$ passes over the data, and which uses space $s$ implies a protocol for $\textsc{DISJ}_{t,d}$ with communication complexity $s\cdot p \cdot t$ 
	\end{lemma} 
	
\begin{proof}
	Given an $\textsc{DISJ}_{t,d}$ instance, we create a matrix $\Ab$ with $d$ columns as follows: 
	Let $e_i$ be  the $i$th row of the $(d \times d)$ identity matrix  $\Ib_{d\times d}$. The vector $e_i$ is associated with the $i$th element of the universe $[d]$. Each player $j$ prepares a matrix $\Ab_j$ with $d$ columns by adding the row $e_i$ for each  $i \in [d]$ in its set.  $\Ab$ is composed of the rows of the $t$ matrices $\Ab_j$.

We claim that $\sqrt t$ approximation to the leverage scores of $\Ab$ suffices to differentiate between the case the sets are disjoint and the case they are uniquely intersecting. To see this note that if the sets are all disjoint then each row is linearly independent from the rest and therefore all rows have leverage score $1$. If the sets are \emph{uniquely intersecting}, then exactly one row in $\Ab$ is repeated $t$ times. A moment's reflection reveals that in this case, each of these rows has leverage score $1/t$. Hence a $\sqrt{t}$-approximation to all leverage scores allows the parties to distinguish between the two cases.

The actual communication protocol is now straight forward. Each party $i$ in turn runs the algorithm over its own matrix $\Ab_i$ and passes the $s$ bits which are the state of the algorithm to the party $i+1$. The last party outputs the result. If the algorithm requires $p$ passes over the data the total communication is $p \cdot s \cdot t$. 
\end{proof}
Theorem~\ref{thm:lower_bound} now follows directly from the lower bound on the communication complexity of $\textsc{DISJ}_{t,d}$.

Note that in the construction above the stable rank $\sr(\Ab)$ is $\Theta(d)$. The dependency on $d$ could be avoided by adding a row and column to $\Ab$:  A column of all zeros is added to $\Ab$ and then the last party adds a row to $\Ab_t$ having the entry $\sqrt{K}\ge 1$ in the last column. Now since $K>1$ the last row will dominate both the Frobenius and the operator norm of the matrix but does not affect the leverage score of the other rows. Note that $\sr(\Ab)\le \frac{K+d}{K}$. By choosing $K$ large enough, we can now decrease $\sr(\Ab)$ to be arbitrarily close to $1$.	
Note also that if the algorithm is restricted to one pass, the resulting protocol is one directional and has a slightly higher lower bound of $\Omega(d/t^2)$.  
	
\paragraph{$(2)$ Ridge leverage scores: } We use the same construction as before and multiply $\Ab$ by $\sigma \geq \sqrt \lambda$. Note that as required, the matrix $\Ab$ has operator norm $\sigma$.
As before, it is sufficient to claim that by approximately computing all ridge leverage scores the parties can distinguish between the case their sets are mutually disjoint and the case they are uniquely intersecting. Indeed, if the sets are mutually disjoint then all rows have ridge leverage scores $\frac{\sigma^2}{\sigma^2+\lambda}$.  If the sets are {uniquely intersecting}, then exactly one element is repeated $t$ times in the matrix $\Ab$, in which case this element has ridge leverage score $\frac{\sigma^2}{t\sigma^2+\lambda}$. These two cases can be distinguished by a $\sqrt{t/2}$-approximation to the ridge leverage scores if $\lambda\le \sigma^2$.  
	
To modify the stable rank $\sr(\Ab)$ in this case we do the same trick as before, add a column of zeros and the last party adds an additional row having the entry $\sqrt{K}\sigma\ge \sigma$ in its last column. Note that $\sr(\Ab)\le\frac{K+d}{K}$, and by increasing $K$ we can decrease $\sr(\Ab)$ as necessary. However, $\norm{\Ab}^2$ now equals ${K}\sigma^2$, and hence we need to upper bound $K$ in terms of the stable rank $\sr(\Ab)$ to state the final bound for $\lambda$ in terms of $\norm{\Ab}^2$. Note that $\sr(\Ab)\le \frac{K+d}{K} \implies K\le 2d/\sr(\Ab)$. Hence $\lambda\le \frac{\sr(\Ab)}{2d}\norm{\Ab}^2$ ensures that $\lambda\le \sigma^2$.
	
\paragraph{$(3)$ Rank-$k$ leverage scores:} The construction is similar to the previous ones, with some modifications to ensure that the top $k$ singular vectors are well defined. We set number of parties to be $2$, and let the universe be of size $d' = d-k$, so the matrix is wider that the size of universe by $k$ columns. As before, for $i\leq d'$ the $i$'th row of  $\Ib_{d\times d}$ is associated with the $i$'th element of the universe $[d']$. 
	The first set of rows in $\Ab$ are the rows corresponding to the elements in the first party's set and the next set of rows in $\Ab$ correspond to the elements in the second party's set. The second party also adds the last $k$ rows of $\Ib_{d\times d}$, scaled by $1.1$, to the matrix $\Ab$.  
	
	We claim that by computing a multiplicative approximation to all rank $k$ leverage scores the parties can determine whether their sets are disjoint. If the sets are all disjoint, then the top $k$ right singular vectors correspond to the last $k$ rows of the matrix $\Ib_{d\times d}$, and these are orthogonal to the rest of the matrix and hence the rank $k$ leverage scores of all rows except the additional ones added by the second party are all $0$. If the sets are  intersecting, then the row corresponding to the intersecting element is the top right singular vector of $\Ab$, as it has singular value $\sqrt{2} > 1.1$. Hence the rank $k$ leverage score of this row is $1/2$. Hence the parties can distinguish whether they have disjoint sets by finding any multiplicative approximation to all rank $k$ leverage scores.
	\remove{
	As before the reduction could be extended to the case where the matrix $\Ab$ has any number of rows by repeating the $(d'+1)$th row of $\Ib_{d\times d}$ any number of times as necessary with the appropriate scaling constant such that the singular value corresponding to that row is still $1.1$. Note that if the sets intersect, the row corresponding to the intersecting element is still the top right singular vector of $\Ab$, and hence a multiplicative approximation suffices to distinguish whether the sets intersect or not.
}

	We apply a final modification to decrease the stable rank $\sr(\Ab)$ as necessary. We scale the $d$th row of $\Ib_{d\times d}$ by a constant $\sqrt{K}$. Note that $\sr(\Ab)\le \frac{K+d+2k}{K}$. By choosing $K$ accordingly, we can now decrease $\sr(\Ab)$ as desired. We now examine how this scaling affects the rank $k$ leverage scores for the rows corresponding to the sets. When the sets are not intersecting, the rank $k$ leverage scores of all the rows corresponding to the set elements are still 0. When the sets are intersecting, the row corresponding to the intersecting element is at the least second largest right singular vector of $\Ab$ even after the scaling, as $\sqrt{2}>1.1$. In this case, for $k\ge2$ the rank $k$ leverage score of this row is 1/2, hence the parties can distinguish whether they have disjoint sets by finding any multiplicative approximation to all rank $k$ leverage scores for any $2\le k\le d/2$.
	
	\paragraph{Distance from principal $k$-dimensional subspace:} We use the same construction as in statement $(3)$. If the sets are non-intersecting, all the rows corresponding to the sets of the two-parties have distance $1$ from the principal $k$-dimensional subspace. If the sets are intersecting, the row corresponding to the  element in the intersection has distance $0$ from the principal  $k$-dimensional subspace, as that row is either the top or the second right singular vector of $\Ab$. Hence, any multiplicative approximation could be used to distinguish between the two cases.
	
\end{proof}

\end{document}